\documentclass[conference]{IEEEtran}

\usepackage{cite}
\usepackage{amsmath,amssymb,amsfonts}
\usepackage{amsthm}
\usepackage{graphicx}
\usepackage{textcomp}
\usepackage{xcolor}
\usepackage{algorithmic}
\usepackage{algorithm}
\usepackage{booktabs}
\usepackage{hyperref}
\usepackage{enumitem}
\usepackage{tikz}
\usepackage{microtype}
\usetikzlibrary{shapes.misc, positioning, arrows.meta, fit}

\newtheorem{definition}{Definition}
\newtheorem{theorem}{Theorem}
\newtheorem{proposition}{Proposition}
\newtheorem{corollary}{Corollary}

\newtheorem{remark}{Remark}

\begin{document}

\title{Toward Standardized Cross-Vendor Agent Tool Trust Management in Autonomous Networks}

\author{\IEEEauthorblockN{Ravi Kant Sharma, Ashutosh Uttam, Ajay Kumar}
\IEEEauthorblockA{Ericsson\\
\{ravi.k.kant.sharma, ashutosh.uttam, ajay.o.kumar\}@ericsson.com}}

\maketitle

% ==============================================================================
% A Standardized Information Model for Cross-Vendor Agent Tool Trust Management
% in Autonomous Networks
% ==============================================================================
% Abstract and Section I: Introduction
% Target: IEEE Transactions on Network and Service Management (TNSM)
% ==============================================================================

\begin{abstract}
Autonomous Network Levels 4--5 require AI agents to invoke tools across vendor boundaries without human oversight, yet existing management standards lack a standardized mechanism for cross-vendor trust visibility. When a tool from Vendor~B is compromised, agents from Vendor~A continue invoking it---unaware of the trust degradation---causing cascading service impact. We present \texttt{AgentToolMO}, a proposed 3GPP NRM information model for agent tool trust management. The model comprises: a formally defined trust state machine with provable graduated enforcement, damped cascade propagation with bounded convergence, cross-vendor trust notifications via existing Management Services (MnS) interfaces, and retroactive impact assessment through NRM dependency graph traversal. Simulation-based evaluation across multi-vendor topologies shows that standardized cross-vendor notifications reduce blast radius from hours-scale undetected propagation to near-real-time containment bounded by MnS notification delivery, with cascade convergence guaranteed in bounded iterations and sub-linear notification scaling across vendor domains. The framework operates within existing 3GPP management infrastructure, leverages existing protocols, and provides a standardization pathway for trustworthy multi-vendor autonomous network management.
\end{abstract}

\begin{IEEEkeywords}
Autonomous networks, AI agents, trust management, network resource model, multi-vendor management, 5G/6G, information model, state machine, 3GPP
\end{IEEEkeywords}

% ==============================================================================
\section{Introduction}
\label{sec:introduction}
% ==============================================================================

%  -  Paragraph 1: Context - Autonomous Networks and AI Agents  - 

Autonomous 5G/6G network management increasingly relies on AI agents that invoke specialized tools across vendor boundaries~\cite{3gpp_tr28.908, 3gpp28105}. In a typical multi-vendor deployment, a radio optimization agent from one vendor may consume a configuration analytics tool managed by another~\cite{3gpp_ts28.533}. These cross-vendor tool dependencies create trust boundaries that are not explicitly addressed by current management frameworks.

%  -  Paragraph 2: The Trust Problem  - 

The core challenge is \textit{trust visibility across vendor boundaries}. A tool's behavior may degrade---its configuration recommendations may drift beyond operator-approved baselines, it may access managed elements outside its declared scope, its actions may cause KPI regression detected through 3GPP performance management, or it may produce inconsistent outputs due to model drift~\cite{huang_survey_2024, kumar_trust_2023}. Within a single vendor domain, such degradation is detected and contained internally. In multi-vendor environments, however, this containment fails: when Vendor~B restricts a degraded tool, Vendor~A's agents - which consume that tool's outputs - have no mechanism to learn of the restriction. We term this \textit{silent trust degradation}. The consequence is that unreliable outputs continue to drive autonomous decisions across domain boundaries, with blast radius expanding undetected until service impact becomes visible~\cite{nist_ai_rmf_2023}.

%  -  Paragraph 3: Limitations of Existing Approaches  - 

Existing approaches fail to close this gap. Sandbox-based trust~\cite{nokia_sandbox_wo2021069196} confines untrusted functions but does not expose trust outcomes cross-vendor. Per-vendor frameworks~\cite{3gpp_tr28.908} treat trust as proprietary internal state. Zero-trust architectures~\cite{nist_sp800207} verify identity and authorization - binary access decisions - but do not assess ongoing behavioral trustworthiness with graduated enforcement.

From a standardization perspective, the 3GPP NRM~\cite{3gpp_ts28.622, 3gpp_ts28.623} provides IOCs for managed entities, and recent extensions introduce AI/ML objects (\texttt{MLTrainingFunction}, \texttt{MLEntity})~\cite{3gpp28105}. However, no IOC exists for agent tool trust state, no MnS notification type covers trust transitions~\cite{3gpp28532}, and no formal safety properties (absence of unsafe state jumps, bounded cascade convergence) have been addressed within the 3GPP context. Prior trust propagation work~\cite{josang_trust_2007, kamvar_eigentrust_2003} provides theoretical foundations but lacks integration with telecom management information models.

%  -  Paragraph 4: Our Contributions  - 

In this paper, we address the cross-vendor agent tool trust visibility problem through a standardized information model approach. Our contributions are as follows:

\begin{enumerate}
\item[\textbf{C1:}] A standardized information model (\textit{AgentToolMO}) that integrates identity, lifecycle, and trust attributes within the 3GPP NRM, enabling cross-vendor queryability of agent tool trust state through existing Management Services interfaces. The model defines a new Information Object Class with attributes, relationships, and notification types that align with established NRM design patterns, ensuring backward compatibility and incremental adoptability.

\item[\textbf{C2:}] A formal trust state machine with graduated enforcement levels and proven safety properties. We demonstrate that the state machine guarantees no unsafe state transitions (a tool cannot jump from a restricted state to full trust without passing through monitored rehabilitation), provides bounded transitions to revocation (at most $k$ intermediate states from any degraded state to full revocation), and maintains monotonic trust reduction under sustained policy violations or KPI regression.

\item[\textbf{C3:}] A trust cascade propagation mechanism with a configurable damping factor that provides mathematical guarantees of convergence. When a tool's trust degrades, dependent tools and agents must be notified; our mechanism ensures that this propagation converges in bounded iterations regardless of topology, prevents cascade storms through exponential attenuation, and respects administrative domain boundaries through configurable propagation policies.

\item[\textbf{C4:}] A retroactive impact assessment algorithm that traverses NRM dependency graphs to identify services and configurations that were influenced by a tool during a period when its trust was subsequently found to be degraded. The algorithm operates within a configurable lookback window and produces a prioritized list of potentially affected entities requiring review or rollback.

\end{enumerate}

%  -  Paragraph 5: Paper Organization  - 

The remainder of this paper is organized as follows. Section~\ref{sec:related} surveys related work in trust management, autonomous network architectures, and relevant standardization efforts. Section~\ref{sec:problem} formally defines the cross-vendor trust visibility problem and establishes requirements for a solution. Section~\ref{sec:framework} presents the \textit{AgentToolMO} information model, including its attributes, relationships, state machine, and integration with the 3GPP NRM. Section~\ref{sec:formal_analysis} provides formal analysis of the model's safety properties, cascade convergence guarantees, and computational complexity bounds. Section~\ref{sec:evaluation} evaluates the approach through simulation across representative multi-vendor topologies, measuring blast radius reduction, cascade convergence time, and detection latency. Section~\ref{sec:discussion} discusses deployment considerations, standardization pathway, limitations, and future directions. Section~\ref{sec:conclusion} concludes the paper.

% Section II and III for:
% A Standardized Information Model for Cross-Vendor Agent Tool Trust Management in Autonomous Networks
% IEEE Transactions on Network and Service Management

\section{Related Work}
\label{sec:related}

The management of trust in autonomous network systems draws upon several distinct research communities. We organize the related literature into five categories, identifying the specific gap that motivates our contribution in each.

\subsection{Trust Management in Multi-Agent Systems}

Computational trust models have been studied extensively. Marsh~\cite{marsh1994} formalized trust as a continuous value derived from past experiences. J{\o}sang~\cite{josang2001} introduced subjective logic for reasoning about trust under uncertainty. Ramchurn et al.~\cite{ramchurn2004} surveyed trust mechanisms including individual-level learning, institutional mechanisms, and socio-cognitive models. More recent work has explored hybrid approaches~\cite{teacy2006}, game-theoretic trust~\cite{noorian2010}, and blockchain-based decentralized trust~\cite{calvaresi2019}.

Despite their maturity, these models assume homogeneous platforms where trust information is internally accessible. None address \emph{standardized information models} enabling trust to be queried across vendor boundaries through well-defined management interfaces - a prerequisite for multi-vendor telecom management.

\subsection{3GPP Management Framework and Information Models}

The 3GPP NRM (TS~28.541~\cite{3gpp28541}) defines Information Object Classes (IOCs) for managed entities accessed through Generic Management Services (TS~28.532~\cite{3gpp28532}). SA5 has extended this for autonomous operations: Management Data Analytics (MDAS)~\cite{3gpp28104}, intent-driven management~\cite{3gpp28312}, and AI/ML management (TS~28.105~\cite{3gpp28105}) introducing \texttt{MLTrainingFunction} and \texttt{MLEntity} IOCs.

Critically, no existing IOC in the 3GPP NRM represents the trust state of an agent's tool. While the framework provides the architectural patterns for our solution - IOC definition, attribute specification, notification mechanisms, and MnS query interfaces - the semantic content for agent tool trust has not been specified.

\subsection{AI/ML Trust and Safety in Telecom Networks}

Several industry initiatives address trust and safety for AI functions in telecommunications, though none provide a cross-vendor standardized information model.

\subsubsection{ETSI ENI}
The European Telecommunications Standards Institute (ETSI) Experiential Networked Intelligence (ENI) Industry Specification Group~\cite{etsi_eni} has defined a reference architecture for AI-assisted network management that includes governance concepts applicable to trust assessment. However, the focus is on a single operator's domain rather than cross-vendor visibility.

\subsubsection{ETSI ZSM}
The ETSI Zero-touch network and Service Management (ZSM) framework~\cite{etsi_zsm_2022} defines a reference architecture for fully automated network management using closed-loop automation across management domains. While ZSM addresses cross-domain coordination through domain integration fabrics and data services, it does not define trust management for the AI-driven decision agents within those domains. The proposed framework complements ZSM by providing the trust governance layer that enables safe agent operation within ZSM-managed domains.

\subsubsection{O-RAN Alliance}
The Open RAN Alliance has specified ML model lifecycle management~\cite{oran_aiml} including model training, validation, deployment, and monitoring phases. The O-RAN architecture supports model performance tracking, but trust assessment remains a local concern within the near-RT RIC or non-RT RIC without standardized exposure to external management systems.

\subsubsection{Cognitive Function Sandboxing}
A sandboxing orchestration framework~\cite{nokia_sandbox_wo2021069196} introduces graduated trust levels for cognitive functions, where untrusted functions execute in sandbox environments and progressively earn permission to act on the live network under specific conditions (temporal, topological, or scope restrictions). This demonstrates the value of graduated trust with behavioral evidence but operates within a single vendor's management entity without cross-vendor visibility.

\subsubsection{Exploration Management for Safe RL Agents}
An exploration management framework~\cite{ericsson_exploration_wo2024183933} coordinates multiple reinforcement learning agents by computing differentiated exploration policies based on operator-defined safety criteria and network part criticality. The framework provides regret guarantees but keeps safety models internal to the exploration management function - not exposed as a queryable information model.

\subsubsection{Network Control Function}
A Network Control Function (NCF)~\cite{nokia_ncf_gb2640186} evaluates cognitive function recommendations before network application using adaptive acceptance criteria (inversely proportional to proposed change magnitude, directly proportional to expected improvement). Trust builds as the function demonstrates reliable predictions. However, the NCF operates as a single-vendor supervisory entity.

\subsubsection{3GPP SA5 Closed Control Loop Trust}
The 3GPP management services for communication service assurance (TS~28.535)~\cite{3gpp28535} includes closed control loop concepts with trust development provisions where loops negotiate precedence on shared resources. However, trust remains implicit in coordination priority rather than an explicit, queryable state machine.

\subsection{Zero-Trust Architectures for Network Management}

Zero-trust architecture (ZTA)~\cite{nist_sp800207} replaces perimeter-based trust with continuous verification, requiring every access request to be authenticated and authorized regardless of network location. CISA's maturity model~\cite{cisa_ztmm} defines progressive trust levels for identity, devices, and data. The NIST AI Risk Management Framework~\cite{nist_ai_rmf_2023} provides guidelines for trustworthy AI systems including governance, mapping, measurement, and management functions - but addresses organizational risk posture rather than runtime trust state machines with cross-vendor notification.

However, ZTA addresses a fundamentally different problem. It focuses on \emph{identity verification} (is this entity who it claims?) and \emph{authorization} (does it have permission?) - binary decisions at each request. Agent tool trust, by contrast, requires \emph{graduated behavioral assessment} over time. A tool may be fully authenticated and authorized yet exhibit configuration drift (recommending parameters outside approved ranges), scope violations (accessing managed elements beyond its declared domain), or recommendation inconsistency (producing conflicting outputs under identical conditions) that warrant trust reduction without complete revocation. Our problem requires a state machine with multiple trust states, transition conditions based on behavioral evidence, and enforcement policies that modulate rather than eliminate tool usage. Furthermore, the trust state must be visible across vendor boundaries - a requirement entirely absent from current zero-trust frameworks.

\subsection{Agent Identity and Authorization Frameworks}

The Model Context Protocol (MCP)~\cite{mcp_spec} provides a universal interface for agents to invoke tools but ``is not a security protocol by design''~\cite{auth0_ai_identity} and offloads authorization to external servers. A knowledge source access framework~\cite{nokia_knowledge_cn122001607} addresses discovery and authorization of knowledge sources using access tokens and consumer profiles in the 3GPP context.

The 3GPP Common API Framework (CAPIF, TS~23.222) provides API registration, discovery, and access control for service APIs. While tools could be viewed as APIs, CAPIF addresses a fundamentally different concern: it governs \emph{whether} an invoker may call an API (binary access decisions at invocation time), whereas our framework governs \emph{how much to trust} an API's outputs based on continuous behavioral observation. CAPIF operates in the service/control plane with no mechanism for graduated trust states, cascade propagation to downstream consumers, or retroactive impact assessment. Our model operates in the management plane (NRM/MnS) and complements CAPIF: a tool may pass all CAPIF authorization checks yet warrant reduced trust due to behavioral drift detected through KPI monitoring.

These frameworks address authentication, authorization, and access control---all prerequisites for, but distinct from, \emph{behavioral trust}. A tool may be fully authorized yet warrant reduced trust due to observed behavioral degradation. Our contribution fills this gap: a standardized information model for ongoing behavioral trust with a formal state machine and cross-vendor notification.

\subsection{Summary of Gaps}

The common limitation across all existing approaches is that trust assessment remains per-vendor: each system uses proprietary criteria, maintains internal trust state, and does not expose it through standardized interfaces. No existing work combines standardized information models, formal state machines with safety guarantees, cross-vendor queryability, and 3GPP NRM integration into a unified framework for agent tool trust.

\section{Problem Statement and Motivation}
\label{sec:problem}

\subsection{The Cross-Vendor Trust Visibility Problem}

Consider a representative multi-vendor 5G network deployment. Vendor~A's management system deploys an autonomous Agent~$\alpha$ responsible for optimizing radio resource management across a set of \texttt{NetworkSliceSubnet} instances. Agent~$\alpha$ accomplishes its objectives by invoking Tool~$T$, a configuration analytics tool provided by Vendor~B's domain management system through a cross-domain MnS interface.

The following scenario illustrates the trust visibility problem:

\begin{enumerate}
    \item Agent~$\alpha$ (Vendor~A) routinely invokes Tool~$T$ (Vendor~B) to obtain configuration recommendations for slice capacity adjustments.
    \item Vendor~B's internal monitoring detects degraded behavior in Tool~$T$: the tool's cell parameter recommendations have drifted outside the approved configuration baseline (handover thresholds exceeding operator-defined bounds), and two recommendations triggered KPI regression alarms in Vendor~B's own test environment.
    \item Vendor~B's management system appropriately degrades Tool~$T$'s internal trust state, restricting its usage within Vendor~B's domain to low-risk operations only.
    \item \textbf{Problem:} Vendor~A's management system has \emph{no visibility} into this trust state change. No notification is generated, no queryable attribute is updated, and no cross-vendor mechanism exists to communicate the degradation.
    \item \textbf{Consequence:} Agent~$\alpha$ continues invoking Tool~$T$ at full operational scope. The degraded tool produces suboptimal or potentially unsafe configuration changes that propagate through the NRM dependency graph - from \texttt{NetworkSliceSubnet} to \texttt{NetworkSlice} to \texttt{ManagedElement} - without any trust-based containment.
\end{enumerate}

This scenario reflects emerging operational patterns. As autonomous management functions proliferate across vendor boundaries - driven by 3GPP's intent-driven management framework and the adoption of AI/ML for network operations - the frequency of cross-vendor tool dependencies will increase substantially. Each such dependency creates an invisible trust boundary where safety-critical information cannot propagate.

\noindent\textit{Running Example.} We illustrate the framework throughout this paper using a concrete scenario: \texttt{CellConfigOptimizer}, a Vendor~B tool that recommends antenna tilt and transmit power parameters for RAN cells. Vendor~A's \texttt{SliceOrchestrator} agent consumes these recommendations to optimize network slice performance. When \texttt{CellConfigOptimizer} receives a faulty model update, its recommendations begin drifting from approved baselines---but Vendor~A has no visibility into this degradation.

\subsection{Quantifying the Gap}

We characterize the severity of the trust visibility gap along three dimensions.

\subsubsection{Detection Latency}
Without a standardized trust model, detection time is effectively unbounded - relying on eventual service degradation visible through KPIs, which may take hours. With standardized notification: $t_{\text{detect}} - t_{\text{event}} \leq \Delta t_{\text{notification}}$, bounded by MnS delivery latency (sub-second).

\subsubsection{Blast Radius}
During the detection gap, the blast radius grows as the tool continues operating. A tool affecting \texttt{NRCellDU} parameters cascades through the NRM: $\texttt{NRCellDU} \rightarrow \texttt{GNBDUFunction} \rightarrow \texttt{NetworkSliceSubnet} \rightarrow \texttt{NetworkSlice}$. Without containment, $|B(t)| = |B_0| \cdot (1 + r \cdot (t - t_{\text{event}}))$ where $r$ is the invocation rate.\footnote{This linearization provides a lower bound; actual growth depends on NRM graph fan-out and is evaluated through simulation in Section~\ref{sec:evaluation}.} With trust-based containment: $|B(t)| \leq |B_0| \cdot (1 + r \cdot \Delta t_{\text{notification}}) \approx |B_0|$.

\subsubsection{Multi-Vendor Scaling}
With $n$ vendors, bilateral trust exchange requires $O(n^2)$ proprietary integrations ($n(n-1)/2$ pairwise agreements). A standardized model reduces this to $O(n)$: each vendor implements the standard IOC once. For $n = 5$, this reduces from 10 bilateral integrations to 5 standard implementations - a $2\times$ reduction that grows linearly with vendor count.

\subsection{Requirements for a Solution}

Based on the analysis of the trust visibility gap and informed by the architectural patterns established in existing work, we derive seven requirements that a solution must satisfy.

\begin{itemize}[leftmargin=*]
    \item \textbf{R1: Standardized Information Model.} The trust state of agent tools must be represented as Information Object Classes queryable via existing 3GPP MnS interfaces (TS~28.532). The model must be expressible in the NRM information modeling methodology and integrate with existing IOC hierarchies without requiring modifications to the base NRM.

    \item \textbf{R2: Formal Trust State Machine.} Trust transitions must be governed by a formally specified state machine with \emph{safety guarantees}: no sequence of inputs may cause a transition from a restricted trust state to an unrestricted state without positive evidence of trustworthiness exceeding defined thresholds.

    \item \textbf{R3: Cross-Vendor Notification.} Trust state changes must generate notifications consumable by any authorized management system through the standardized 3GPP notification mechanism.

    \item \textbf{R4: Cascade Containment with Convergence.} When trust degradation in one tool triggers trust reassessment of dependent tools, the cascade must be bounded and convergent.

    \item \textbf{R5: Retroactive Impact Assessment.} Upon trust degradation, it must be possible to identify the set of management actions performed by the affected tool during a specified lookback window.

    \item \textbf{R6: Operator Configurability.} Trust thresholds and enforcement policies must be operator-configurable through the standard provisioning MnS without requiring software upgrades.

    \item \textbf{R7: Protocol Independence.} The information model must be independent of specific transport protocols.
\end{itemize}

\subsection{Formal Problem Definition}

We formalize the problem. A multi-vendor management network is a directed graph $G_v = (V, E_v)$ where $V = \{v_1, \ldots, v_n\}$ are vendor management domains and $(v_i, v_j) \in E_v$ indicates agents in $v_i$ use tools from $v_j$. Let $\mathcal{T}$ be the set of all tools with ownership function $\text{owner}: \mathcal{T} \rightarrow V$, and $\text{uses}: \mathcal{A} \rightarrow 2^{\mathcal{T}}$ mapping agents to their tools.

The trust visibility function $\text{vis}: V \times \mathcal{T} \rightarrow \{\textit{visible}, \textit{invisible}\}$ captures whether a domain can observe a tool's trust state. Currently:
\begin{equation}
    \text{vis}(v_i, t_j) = \begin{cases}
    \textit{visible} & \text{if } v_i = \text{owner}(t_j) \\
    \textit{invisible} & \text{otherwise}
    \end{cases}
\end{equation}

The goal is to achieve: $\text{vis}(v_i, t_j) = \textit{visible}$ for all $v_i$ where $\exists\, a \in \mathcal{A}_{v_i} : t_j \in \text{uses}(a)$, where $\mathcal{A}_{v_i}$ denotes the set of agents belonging to vendor domain $v_i$ (i.e., all agents owned and operated within that vendor's management system).

\begin{definition}[Trust State]
For each tool $t_j$, the trust state is a tuple $\sigma(t_j) = (s, \mathbf{e}, \tau, \rho)$ where $s \in S$ is the current state in the trust state machine (defined in Section~\ref{sec:framework}), $\mathbf{e}$ is the evidence vector, $\tau$ is the timestamp of the most recent transition, and $\rho: S \rightarrow \mathcal{P}$ maps each state to its enforcement policy.
\end{definition}

\noindent\textbf{Problem Statement.} Design an information model $\mathcal{I}_M$ such that: (1) $\mathcal{I}_M$ is expressible as a 3GPP-compliant IOC enabling cross-vendor visibility; (2) the trust state machine satisfies safety - no transition from a degraded state to a higher-trust state without sufficient evidence; (3) trust state changes generate cross-vendor notifications with latency $\leq \Delta t_{\text{max}}$; (4) any cascade converges within $k_{\text{max}}$ steps.

The following sections present our solution: Section~\ref{sec:framework} defines the information model including the trust state machine and cascade containment algorithm, and Section~\ref{sec:evaluation} provides quantitative evaluation.

\section{Proposed Framework}
\label{sec:framework}

This section presents the complete technical framework for cross-vendor agent tool trust management in autonomous networks. The framework is grounded in the 3GPP Network Resource Model (NRM) architecture defined in TS~28.541~\cite{3gpp28541} and extends it with managed objects, state machines, and algorithms specifically designed for governing agent tool interactions in multi-vendor network management environments.

\begin{figure}[t]
\centering
\resizebox{\columnwidth}{!}{%
\begin{tikzpicture}[
    every node/.style={font=\footnotesize},
    ioc/.style={draw, rectangle, rounded corners=2pt, minimum width=2.6cm, minimum height=0.6cm, align=center, line width=0.6pt, fill=white},
    newioc/.style={draw, rectangle, rounded corners=2pt, minimum width=2.6cm, minimum height=0.6cm, align=center, line width=0.8pt, fill=blue!5, draw=blue!70!black},
    contain/.style={-{Triangle[length=2.5mm, width=2mm, open]}, line width=0.5pt},
    assoc/.style={-{Stealth[length=2mm]}, line width=0.5pt, dashed, blue!70!black}
]

% Row 0
\node[ioc] (top) at (3, 0) {\texttt{Top}};

% Row 1
\node[ioc] (me) at (0, -2) {\texttt{ManagedElement}};
\node[ioc] (sub) at (6, -2) {\texttt{SubNetwork}};

% Row 2
\node[ioc] (gnb) at (0, -4) {\texttt{GnbDuFunction}};
\node[newioc] (agent) at (6, -4) {\texttt{AgentMO}};

% Row 3
\node[newioc] (tool) at (4, -6.5) {\texttt{AgentToolMO}};
\node[newioc] (policy) at (8, -6.5) {\texttt{AgentPolicyMO}};

% Containment arrows (all straight vertical)
\draw[contain] (me) -- (top);
\draw[contain] (sub) -- (top);
\draw[contain] (gnb) -- (me);
\draw[contain] (agent) -- (sub);
\draw[contain] (tool) -- (agent);

% governedBy: horizontal with label placed as standalone node above midpoint
\draw[assoc] (tool) -- (policy);
\node[font=\scriptsize, text=blue!70!black] at (6, -6.1) {\textit{governedBy}};

% appliesTo: orthogonal path - go left then up
\draw[assoc] (tool.west) -- (2, -6.5) -- (2, -4.6) -- (gnb.south);
\node[font=\scriptsize, text=blue!70!black] at (2, -5.5) [left] {\textit{appliesTo}};

% Bounding box for proposed IOCs
\node[draw=blue!70!black, dashed, rounded corners=3pt, inner sep=10pt, fit=(agent)(tool)(policy), label={[font=\scriptsize, text=blue!70!black]below:Illustrative IOCs}] {};

% Legend (single row at bottom)
\node[anchor=north west, font=\scriptsize] at (-1, -8) {%
\begin{tabular}{@{}l@{\hspace{4pt}}l@{\hspace{14pt}}l@{\hspace{4pt}}l@{}}
\tikz[baseline=-0.5ex]\draw[-{Triangle[length=2.5mm, width=2mm, open]}, line width=0.5pt](0,0)--(0.6,0); & Containment &
\tikz[baseline=-0.5ex]\draw[-{Stealth[length=2mm]}, dashed, blue!70!black, line width=0.5pt](0,0)--(0.6,0); & Reference \\
\end{tabular}};
\end{tikzpicture}%
}
\caption{Illustrative IOC placement within the 3GPP NRM hierarchy. Blue-bordered classes are new; solid triangular arrows denote containment (parent--child); dashed arrows denote reference associations.}
\label{fig:nrm_hierarchy}
\end{figure}
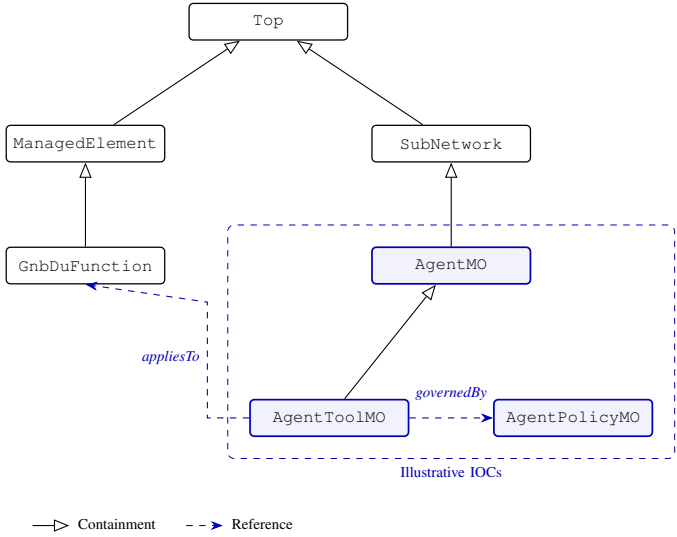

\subsection{Core Managed Object for Agent Tool Trust}
\label{sec:framework:mo}

The specific attribute names, state labels, threshold values, and algorithmic details presented in this section represent one illustrative instantiation of the framework concepts. To delineate scope: the \emph{normative contributions} of this work are (i)~the trust state machine structure and its safety properties, (ii)~the cascade propagation mechanism with damped convergence, (iii)~the cross-vendor notification architecture, and (iv)~the retroactive impact assessment approach. The specific attribute names (e.g., \texttt{trustScore}, \texttt{anomalyCount}), penalty formulas, threshold values, and enumeration types are \emph{illustrative} and would be subject to refinement during any standardization process. The specific attribute naming (e.g., \texttt{toolId}, \texttt{trustScore}) follows 3GPP NRM conventions (TS~28.622 naming patterns using lowerCamelCase) but alternative realizations with different granularity or parameterization are equally consistent with the underlying architectural principles.

We propose \texttt{AgentToolMO} as an illustrative first-class managed object within the 3GPP NRM, representing any tool that an agent may invoke to observe or modify network state (Figure~\ref{fig:nrm_hierarchy} shows its placement in the NRM hierarchy). This object captures the complete lifecycle of a tool from onboarding through potential deregistration, with continuous trust assessment (Figure~\ref{fig:deployment_arch}). Throughout this paper, we use the term \emph{tool} as shorthand for an \texttt{AgentToolMO} instance; the two terms are used interchangeably unless a distinction between the managed object class and a specific instance is contextually relevant.

\begin{figure}[t]
\centering
\resizebox{\columnwidth}{!}{%
\begin{tikzpicture}[
    every node/.style={font=\footnotesize},
    box/.style={draw, rectangle, rounded corners=2pt, minimum height=0.6cm, align=center, line width=0.6pt},
    vendordom/.style={draw, rectangle, rounded corners=4pt, minimum width=3.6cm, minimum height=2.4cm, line width=0.8pt, fill=gray!4},
    layer/.style={draw, rectangle, rounded corners=2pt, minimum width=10.5cm, minimum height=0.7cm, line width=0.8pt, align=center},
    arr/.style={-{Stealth[length=2mm]}, line width=0.5pt},
    notif/.style={-{Stealth[length=2mm]}, line width=0.5pt, dashed, red!70!black}
]

% Top: MnS layer
\node[layer, fill=green!5, draw=green!60!black] (mns) at (0, 0) {\textbf{MnS Notification Interface} (TS~28.532)};

% Vendor A - left
\node[vendordom] (va) at (-3.5, -2.6) {};
\node[font=\footnotesize\bfseries] at (-3.5, -1.6) {Vendor A (RAN)};
\node[box, fill=blue!5, draw=blue!60!black, minimum width=3.0cm] (ta) at (-3.5, -2.3) {\texttt{AgentToolMO}};
\node[box, fill=white, minimum width=3.0cm] (aa) at (-3.5, -3.0) {5 Agents, 12 Tools};

% Vendor B - right
\node[vendordom] (vb) at (3.5, -2.6) {};
\node[font=\footnotesize\bfseries] at (3.5, -1.6) {Vendor B (Core)};
\node[box, fill=blue!5, draw=blue!60!black, minimum width=3.0cm] (tb) at (3.5, -2.3) {\texttt{AgentToolMO}};
\node[box, fill=white, minimum width=3.0cm] (ab) at (3.5, -3.0) {5 Agents, 10 Tools};

% Arrows from MnS to vendors
\draw[arr] (-2, -0.35) -- (-2, -1.4) node[midway, above left, font=\scriptsize] {notify};
\draw[arr] (2, -0.35) -- (2, -1.4) node[midway, above right, font=\scriptsize] {notify};

% Cross-vendor notification - centered between boxes with plenty of room
\draw[notif] (-1.7, -2.6) -- (1.7, -2.6);
\node[font=\scriptsize, text=red!70!black] at (0, -2.95) {\textit{notifyTrustDegradation}};

\end{tikzpicture}%
}
\caption{Deployment architecture (simplified; evaluation uses 3 vendors). Each vendor domain maintains local \texttt{AgentToolMO} instances. Trust state changes propagate via the MnS notification interface (top) or direct cross-vendor notifications (dashed). No bilateral integration required.}
\label{fig:deployment_arch}
\end{figure}
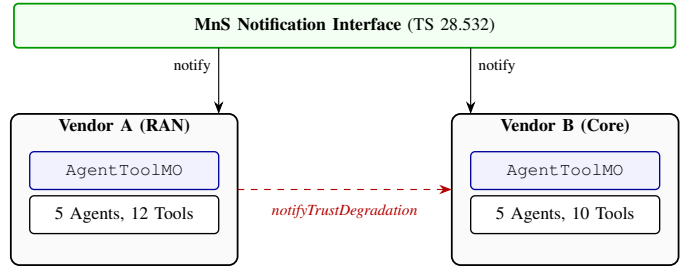

\begin{definition}[AgentToolMO]
\label{def:agenttoolmo}
An \texttt{AgentToolMO} is a managed object $\mathcal{O} = (\mathcal{I}, \mathcal{L}, \mathcal{R})$ where $\mathcal{I}$ is the identity attribute group, $\mathcal{L}$ is the lifecycle attribute group, and $\mathcal{R}$ is the trust attribute group. Each instance is uniquely identified within the NRM and subject to the trust governance mechanisms defined in this framework.
\end{definition}

\subsubsection{Identity Attributes}
The identity attribute group $\mathcal{I}$ provides immutable identification of the tool:

Attribute naming follows 3GPP NRM conventions (TS~28.622~\cite{3gpp_ts28.622}): \texttt{vendorName} reuses the existing NRM attribute for managed elements, while \texttt{toolId} and \texttt{toolName} mirror the \texttt{id}/\texttt{userLabel} pattern established for IOCs such as \texttt{ManagedElement} and \texttt{MLEntity}~\cite{3gpp28105}:

\begin{itemize}
    \item \texttt{toolId}: Globally unique identifier (UUID v4) assigned at onboarding time.
    \item \texttt{toolName}: Human-readable name of the tool (e.g., ``CellConfigOptimizer'').
    \item \texttt{toolDescription}: Free-text description of tool capabilities and intended use.
    \item \texttt{vendorName}: Identifier of the vendor that developed and maintains the tool.
    \item \texttt{toolVersion}: Semantic version string following the format \texttt{MAJOR.MINOR.PATCH}.
    \item \texttt{networkAccessScope}: Enumerated set defining the tool's authorized network access boundaries, drawn from $\{\texttt{RAN}, \texttt{Core}, \texttt{Transport}, \texttt{Slice}, \texttt{CrossDomain}\}$.
\end{itemize}

\subsubsection{Lifecycle Attributes}
The lifecycle attribute group $\mathcal{L}$ tracks the temporal progression of the tool:

\begin{itemize}
    \item \texttt{onboardingTimestamp}: ISO~8601 timestamp recording when the tool was first registered in the NRM.
    \item \texttt{learningCompletionTime}: Timestamp when the tool's initial learning period concluded (null if still in learning).
    \item \texttt{lifecycleState}: Current lifecycle state $\in \{\texttt{Active}, \texttt{Suspended}, \texttt{Archived}\}$.
    \item \texttt{deregistrationTime}: Timestamp of deregistration (null if active).
    \item \texttt{deregistrationReason}: Enumerated reason code $\in \{\texttt{OperatorInitiated},$ $\texttt{TrustRevocation}, \texttt{VersionSuperseded},$ $\texttt{VendorWithdrawal}\}$.
\end{itemize}

\subsubsection{Trust Attributes}
The trust attribute group $\mathcal{R}$ maintains the runtime trust posture:

\begin{itemize}
    \item \texttt{trustScore}: Integer value $s \in [0, 100]$ representing the quantified trust score, where 100 indicates maximum trust.
    \item \texttt{trustState}: Current state in the trust state machine $\in S$ (defined in Section~\ref{sec:framework:statemachine}).
    \item \texttt{previousTrustState}: The state immediately prior to the last transition.
    \item \texttt{trustStateTransitionTime}: Timestamp of the most recent state transition.
    \item \texttt{anomalyCount}: Cumulative count of detected behavioral anomalies since last trust recovery.
    \item \texttt{lastAnomalyTime}: Timestamp of the most recently detected anomaly.
    \item \texttt{lastAnomalyType}: Classification of the most recent anomaly. Values: \texttt{ConfigurationDrift}, \texttt{ScopeViolation}, \texttt{KpiRegression}, \texttt{RecommendationInconsistency}, \texttt{CredentialAnomaly}, \texttt{LatencyBreach}, or \texttt{UnauthorizedDataAccess}.
    \item \texttt{revocationCount}: Cumulative count of trust revocations over the tool's lifetime.
    \item \texttt{policyRef}: Reference (DN) to the governing \texttt{AgentPolicyMO} instance.
\end{itemize}

\subsubsection{Trust Score Computation}
The trust score is computed as a weighted combination of behavioral evidence signals, clamped to the range $[0, 100]$:
\begin{equation}
\label{eq:trustscore}
\texttt{trustScore}(t) = \max\left(0,\; 100 - \sum_{i=1}^{N} w_i \cdot p_i(t)\right)
\end{equation}
where $p_i(t) \in [0, 100]$ is the $i$-th penalty component at time $t$ (scaled to a 0--100 range), $w_i$ is its operator-configurable weight ($\sum w_i = 1$), and $N$ is the number of active penalty sources. The $\max(0, \cdot)$ operator ensures the score remains non-negative even under multiple simultaneous penalty sources; the upper bound of 100 is guaranteed by construction since all penalty components are non-negative.

The penalty components are:

\begin{itemize}
    \item $p_{\text{drift}}$: Magnitude of configuration recommendation deviation from approved baseline, normalized by maximum allowed deviation and scaled to $[0, 100]$.
    \item $p_{\text{scope}}$: Binary penalty ($0$ or $100$) for any access to managed elements outside declared \texttt{networkAccessScope}. Resets to $0$ upon operator acknowledgment or after a configurable quarantine period defined in \texttt{AgentPolicyMO}.
    \item $p_{\text{kpi}}$: Normalized KPI regression magnitude observed after tool-initiated changes, measured as $|\Delta \text{KPI}| / \text{KPI}_{\text{baseline}}$ and scaled to $[0, 100]$.
    \item $p_{\text{inconsistency}}$: Ratio of inconsistent recommendations to total recommendations within a sliding window, scaled to $[0, 100]$.
    \item $p_{\text{anomaly}}$: Exponentially decaying penalty per detected anomaly, scaled to $[0, 100]$: $\min(100,\; \sum_{j} e^{-(t - t_j)/\lambda} \times 100)$ where $t_j$ are anomaly timestamps and $\lambda$ is the decay constant.
\end{itemize}

The score starts at 100 (maximum trust) upon onboarding and decreases as penalties accumulate. Recovery occurs as penalties decay over anomaly-free periods. The weights $w_i$ and decay constant $\lambda$ are configured in \texttt{AgentPolicyMO}, enabling operators to prioritize different evidence types based on their risk posture.

\noindent\textbf{Formal penalty definitions.} The following concrete formulations are provided to demonstrate feasibility and enable reproducibility of the evaluation results; alternative formulations satisfying the same architectural requirements (bounded, non-negative penalties on a $[0, 100]$ scale feeding a weighted aggregation) are equally valid within the framework. Each penalty component $p_i \in [0, 100]$ is precisely defined as:
{\small
\begin{align}
p_{\mathrm{drift}}(t) &= \frac{|\mathbf{r}(t) - \mathbf{b}|}{\delta_{\max}} \times 100 \label{eq:pdrift} \\
p_{\mathrm{scope}}(t) &= \begin{cases} 100 & \text{if } \exists\, me \notin \texttt{networkAccessScope} \\ & \quad\text{accessed at } t \\ 0 & \text{otherwise} \end{cases} \label{eq:pscope} \\
p_{\mathrm{kpi}}(t) &= \frac{|\mathrm{KPI}(t) - \mathrm{KPI}_{\mathrm{pre}}|}{\mathrm{KPI}_{\mathrm{baseline}}} \times 100 \label{eq:pkpi} \\
p_{\mathrm{inconsistency}}(t) &= \frac{|\{r_j \in W_s : r_j \neq r_1\}|}{|W_s|} \times 100 \label{eq:pincon} \\
p_{\mathrm{anomaly}}(t) &= \min\!\left(100,\; \sum_{j=1}^{N_a} e^{-(t - t_j)/\lambda} \times 100\right) \label{eq:panomaly}
\end{align}
}
where $\mathbf{r}(t)$ is the tool's recommended configuration vector, $\mathbf{b}$ is the operator-approved baseline (configured per tool in \texttt{AgentPolicyMO}), $\delta_{\max}$ is the maximum approved deviation, $\mathrm{KPI}_{\mathrm{pre}}$ is the KPI value measured immediately before the tool's action, $W_s$ is a sliding window of the $|W_s|$ most recent recommendations (default $|W_s| = 20$), $r_j$ denotes individual recommendations within the window, $N_a$ is the number of anomaly events with timestamps $t_j$, and all $p_i$ are clamped to $[0, 100]$. With $\sum w_i = 1$, a penalty of $p_i = 100$ at weight $w_i$ contributes $w_i \times 100$ points of trust score reduction (e.g., $w_{\mathrm{drift}} = 0.3$ and $p_{\mathrm{drift}} = 60$ yields an 18-point penalty).

We adopt a linear weighted model for trust score computation for three reasons. First, linearity provides \emph{interpretability}: operators can directly reason about the contribution of each evidence type (e.g., ``a scope violation costs 15 points''), which is essential for operational acceptance and auditability in regulated network management environments. Second, the model supports \emph{composability}: new evidence sources can be added by extending the penalty vector without restructuring existing computations. Third, the configurable weights enable \emph{domain adaptation}: different operational contexts (e.g., safety-critical RAN vs.\ best-effort transport) can reflect different risk postures through weight tuning rather than algorithmic redesign. While nonlinear models (e.g., Bayesian inference, neural trust predictors) could capture inter-penalty correlations, they sacrifice the transparency required for cross-vendor interoperability where all parties must agree on how trust scores are derived.

\noindent\textit{Running Example (continued).} After the faulty model update, \texttt{CellConfigOptimizer}'s $p_{\mathrm{drift}}$ rises from 0 to 60 (recommendations deviate 60\% from baseline, scaled to $[0, 100]$). With default weight $w_{\mathrm{drift}} = 0.3$, the trust score drops from 100 to $100 - 0.3 \times 60 = 82$. A subsequent KPI regression ($p_{\mathrm{kpi}} = 40$, weight 0.25) further reduces the score to $82 - 0.25 \times 40 = 72$. An anomaly event ($p_{\mathrm{anomaly}} = 15$, $w_{\mathrm{anomaly}} = 0.2$, penalty $= 3$) then pushes the score to 69, crossing $\theta_{\mathrm{monitor}} = 70$.

\subsubsection{Penalty Independence and Causal Deduplication}
The five penalty components are designed to measure orthogonal evidence dimensions: $p_{\text{drift}}$ measures output quality, $p_{\text{scope}}$ measures access boundary compliance, $p_{\text{kpi}}$ measures downstream impact, $p_{\text{inconsistency}}$ measures output stability, and $p_{\text{anomaly}}$ measures event frequency. When a single root cause (e.g., a corrupted model update) triggers multiple penalty sources simultaneously, a causal deduplication mechanism prevents over-penalization: if $p_{\text{kpi}}$ and $p_{\text{drift}}$ co-occur within a configurable correlation window $\Delta t_{\text{corr}}$ (default: 60\,s) and share a common triggering action in \texttt{AgentActionLogMO}, only the maximum of the correlated penalties is applied rather than their sum. Formally:
\begin{equation}
\label{eq:dedup}
\texttt{effectivePenalty}(t) = \sum_{g \in \mathcal{G}} \max_{i \in g} \left( w_i \cdot p_i(t) \right)
\end{equation}
where $\mathcal{G}$ partitions the active penalties into causally independent groups based on triggering action correlation. In the absence of causal correlation (the common case), each penalty belongs to a singleton group and the formula reduces to standard linear aggregation. This ensures that a single misconfiguration event cannot drive the trust score to zero through multiplicative penalty stacking.

\noindent\textbf{Correlation grouping procedure.} Penalties are grouped into causal sets $\mathcal{G}$ as follows: two penalty sources $p_i$ and $p_j$ are placed in the same group $g \in \mathcal{G}$ if and only if (a)~both were triggered within $\Delta t_{\mathrm{corr}}$ of a common action $\lambda \in \texttt{AgentActionLogMO}$ (matched by \texttt{actionId}), AND (b)~both reference the same \texttt{targetManagedElement}. The correlation window $\Delta t_{\mathrm{corr}} = 60$\,s is derived from the 95th percentile of KPI measurement lag in typical 3GPP performance management reporting intervals (TS~28.552 specifies 15-second granularity periods; 60\,s accommodates up to 4 reporting cycles for delayed KPI manifestation). Groups are computed greedily: each new penalty event is matched against existing groups within the window; unmatched penalties form singleton groups.

The correlation window is configurable in \texttt{AgentPolicyMO} (attribute \texttt{correlationWindowSeconds}, default 60). Operators should tune based on their KPI collection interval:
\begin{center}
\scriptsize
\begin{tabular}{lc}
\toprule
\textbf{KPI Collection Interval} & \textbf{Recommended $\Delta t_{\mathrm{corr}}$} \\
\midrule
15\,s (3GPP default, TS~28.552) & 60\,s ($4\times$ interval) \\
30\,s (extended reporting) & 120\,s \\
60\,s (legacy systems) & 240\,s \\
\bottomrule
\end{tabular}
\end{center}

\subsubsection{Agent Managed Object}
We introduce \texttt{AgentMO} representing agents that consume tools:

\begin{itemize}
    \item \texttt{agentId}, \texttt{agentName}, \texttt{vendorInfo}: Identity attributes.
    \item \texttt{agentType} $\in \{\texttt{Planning}, \texttt{Execution},$ $\texttt{Monitoring}, \texttt{Orchestration}\}$.
    \item \texttt{protocolsSupported[]}: Supported tool invocation protocols.
    \item \texttt{autonomyLevel} $\in [0,5]$: Per the 3GPP autonomy classification~\cite{3gpp28910}.
    \item \texttt{toolList[]}: References to authorized \texttt{AgentToolMO} instances.
    \item \texttt{networkDomain}, \texttt{operationalState}: Operational context.
\end{itemize}

\subsubsection{Session and Action Tracking}
\texttt{AgentSessionMO} tracks active agent-tool interactions:

\begin{itemize}
    \item \texttt{sessionId}, \texttt{agentRef}, \texttt{toolRef}: Identity and references.
    \item \texttt{sessionStartTime}: When the session began.
    \item \texttt{sessionState} $\in \{\texttt{Active},$ $\texttt{Suspended}, \texttt{Terminated}\}$.
    \item \texttt{invocationCount}, \texttt{lastInvocationTime}: Usage statistics.
\end{itemize}

\texttt{AgentActionLogMO} provides auditability for retroactive assessment:

\begin{itemize}
    \item \texttt{actionId}, \texttt{executorRef}, \texttt{toolRef}: Identity and references.
    \item \texttt{actionType} $\in \{\texttt{Read}, \texttt{Write}, \texttt{Execute}\}$.
    \item \texttt{targetManagedElement}: DN of the affected managed element.
    \item \texttt{executionTime}, \texttt{outcome}: When and what resulted.
    \item \texttt{configChangeDetails}: Before/after values of any configuration change.
\end{itemize}

\subsubsection{Relationships}
The \texttt{AgentToolMO} participates in the following NRM relationships:

\begin{itemize}
    \item \texttt{usedByAgents[]}: A list of DN references to \texttt{AgentMO} instances that are authorized consumers of this tool. This is a bidirectional relationship with \texttt{AgentMO.toolList[]}.
    \item \texttt{appliesTo[]}: A list of DN references to managed entities (e.g., \texttt{ManagedElement}, \texttt{SubNetwork}) over which this tool has operational scope.
    \item \texttt{governedBy}: A DN reference to the \texttt{AgentPolicyMO} that defines the trust governance parameters for this tool.
    \item \texttt{dependsOn[]}: A list of DN references to other \texttt{AgentToolMO} instances upon which this tool has functional dependencies. This relationship is critical for trust cascade propagation (Section~\ref{sec:framework:cascade}).
\end{itemize}

\subsection{Graduated Trust State Machine}
\label{sec:framework:statemachine}

The trust lifecycle of each \texttt{AgentToolMO} is governed by a formally defined finite state machine that ensures graduated enforcement - tools cannot transition directly from full trust to revocation without passing through intermediate states that provide opportunity for corrective action.

\begin{definition}[Trust State Machine]
\label{def:statemachine}
The Trust State Machine is a tuple $\mathcal{M} = (S, \Delta, s_0, \mathcal{F})$ where:
\begin{itemize}
    \item $S = \{\texttt{Learning}, \texttt{Trusted}, \texttt{Monitored},$ $\texttt{Restricted}, \texttt{Revoked}, \texttt{Deregistered}\}$ is the finite set of trust states,
    \item $\Delta \subseteq S \times S \times \Gamma \times \mathcal{G} \times \mathcal{K}$ is the set of allowed transitions, where $\Gamma$ is the set of triggers, $\mathcal{G}$ is the set of guard conditions, and $\mathcal{K}$ is the set of transition actions,
    \item $s_0 = \texttt{Learning}$ is the initial state assigned to every newly onboarded tool,
    \item $\mathcal{F} = \{\texttt{Deregistered}\}$ is the set of terminal (absorbing) states.
\end{itemize}
\end{definition}

The six states capture distinct operational regimes, as depicted in Figure~\ref{fig:state_machine}:

\begin{enumerate}
    \item \textbf{Learning} ($s_L$): The tool has been onboarded and is under initial observation. Its behavior is being profiled to establish baseline patterns. Write and execute operations may be permitted under enhanced monitoring, which counts anomalies for the $\tau_1$ graduation guard evaluation.
    \item \textbf{Trusted} ($s_T$): The tool has demonstrated reliable behavior over the required learning period. Full operational privileges are granted with standard monitoring.
    \item \textbf{Monitored} ($s_M$): Anomalous behavior has been detected. The tool retains operational capabilities but monitoring frequency and depth are increased.
    \item \textbf{Restricted} ($s_R$): Continued or severe anomalies have reduced trust below the restriction threshold. Write and execute operations are suspended; only read operations are permitted.
    \item \textbf{Revoked} ($s_V$): Trust has been fully revoked. All sessions are terminated and no operations are permitted. The tool remains in the NRM for audit purposes.
    \item \textbf{Deregistered} ($s_D$): Terminal state. The tool has been permanently removed from active management. Logs are archived.
\end{enumerate}

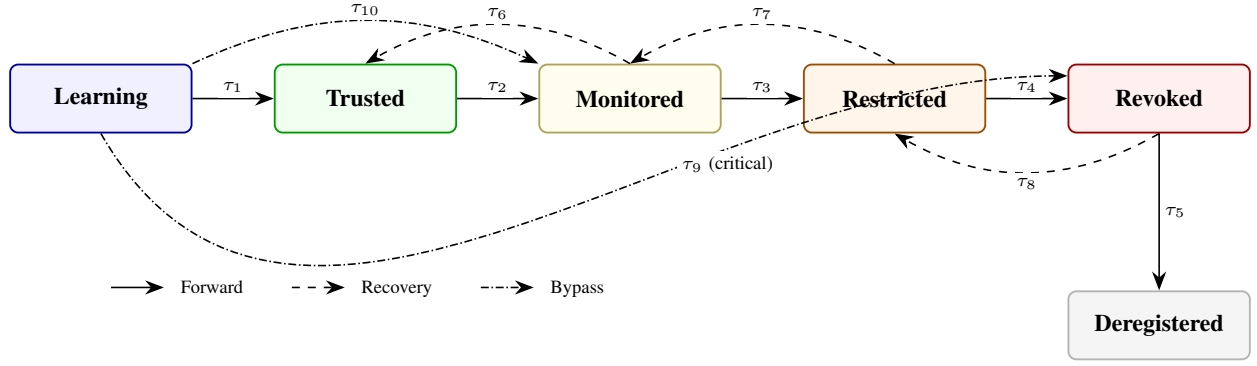
\begin{figure*}[t]
\centering
\begin{tikzpicture}[
    every node/.style={font=\small},
    state/.style={draw, rectangle, rounded corners=3pt, minimum width=2.4cm, minimum height=0.9cm, font=\small\bfseries, align=center, line width=0.7pt},
    forward/.style={-{Stealth[length=2.5mm]}, line width=0.6pt},
    recovery/.style={-{Stealth[length=2.5mm]}, line width=0.6pt, dashed},
    bypass/.style={-{Stealth[length=2.5mm]}, line width=0.6pt, densely dashdotted},
    lbl/.style={font=\scriptsize, fill=white, inner sep=2pt}
]

% Row 1: Main degradation path (left to right)
\node[state, fill=blue!6, draw=blue!60!black] (L) at (0, 0) {Learning};
\node[state, fill=green!6, draw=green!60!black] (T) at (3.5, 0) {Trusted};
\node[state, fill=yellow!8, draw=yellow!60!black] (M) at (7, 0) {Monitored};
\node[state, fill=orange!8, draw=orange!60!black] (R) at (10.5, 0) {Restricted};
\node[state, fill=red!6, draw=red!60!black] (V) at (14, 0) {Revoked};

% Row 2: Terminal state (centered below)
\node[state, fill=gray!8, draw=gray!60] (D) at (14, -3) {Deregistered};

% === Forward transitions (solid) ===
\draw[forward] (L) -- node[lbl, above] {$\tau_1$} (T);
\draw[forward] (T) -- node[lbl, above] {$\tau_2$} (M);
\draw[forward] (M) -- node[lbl, above] {$\tau_3$} (R);
\draw[forward] (R) -- node[lbl, above] {$\tau_4$} (V);
\draw[forward] (V) -- node[lbl, right] {$\tau_5$} (D);

% === Recovery transitions (dashed, curved above) ===
\draw[recovery] (M.north) to[out=150, in=30] node[lbl, above] {$\tau_6$} (T.north);
\draw[recovery] (R.north) to[out=150, in=30] node[lbl, above] {$\tau_7$} (M.north);
\draw[recovery] (V.south) to[out=-150, in=-30] node[lbl, below] {$\tau_8$} (R.south);

% === Bypass transitions (dash-dot) ===
\draw[bypass] (L.south) to[out=-60, in=180] node[lbl, below, pos=0.7] {$\tau_9$\scriptsize~(critical)} ([yshift=0.3cm]V.west);
\draw[bypass] (L.north east) to[out=25, in=155] node[lbl, above] {$\tau_{10}$} (M.north west);

% === Legend (bottom left, clean) ===
\node[anchor=north west] at (0, -2.2) {%
\begin{tabular}{@{}l@{\hspace{6pt}}l@{\hspace{18pt}}l@{\hspace{6pt}}l@{\hspace{18pt}}l@{\hspace{6pt}}l@{}}
\tikz[baseline=-0.5ex]\draw[forward](0,0)--(0.7,0); & \scriptsize Forward &
\tikz[baseline=-0.5ex]\draw[recovery](0,0)--(0.7,0); & \scriptsize Recovery &
\tikz[baseline=-0.5ex]\draw[bypass](0,0)--(0.7,0); & \scriptsize Bypass \\
\end{tabular}};

\end{tikzpicture}
\caption{Trust state machine. Forward transitions ($\tau_1$--$\tau_5$, solid) represent the lifecycle progression: $\tau_1$ promotes from Learning to Trusted; $\tau_2$--$\tau_4$ handle graduated degradation; $\tau_5$ confirms deregistration. Recovery transitions ($\tau_6$--$\tau_8$, dashed) require operator approval and sustained anomaly-free behavior. Bypass transitions ($\tau_9$--$\tau_{10}$, dash-dotted) handle critical violations during learning. \texttt{Deregistered} is a terminal absorbing state with no outgoing transitions.}
\label{fig:state_machine}
\end{figure*}

\subsubsection{Allowed Transitions}
In our illustrative instantiation, we employ ten transitions $\Delta = \{\tau_1, \tau_2, \ldots, \tau_{10}\}$, each specified as a 5-tuple $(s_{\mathrm{src}}, s_{\mathrm{dst}}, \tau, g, a)$. Table~\ref{tab:transitions} summarizes all transitions.

\begin{table*}[t]
\centering
\caption{Trust State Machine Transitions}
\label{tab:transitions}
\scriptsize
\resizebox{\textwidth}{!}{%
\begin{tabular}{@{}clllll@{}}
\toprule
& \textbf{Transition} & \textbf{Trigger} & \textbf{Guard} & \textbf{Key Actions} \\
\midrule
$\tau_1$ & Learning $\rightarrow$ Trusted & learningComplete $\wedge$ anomalies $< \theta_{\mathrm{learn}}$ & $t_{\mathrm{now}} - t_{\mathrm{onboard}} \geq \Delta t_{\mathrm{minLearn}}$ & Set state, emit notification \\
$\tau_2$ & Trusted $\rightarrow$ Monitored & anomalyDetected & trustScore $< \theta_{\mathrm{monitor}}$ & Increase monitoring frequency \\
$\tau_3$ & Monitored $\rightarrow$ Restricted & repeatedAnomalies $\vee$ policyBreach & trustScore $< \theta_{\mathrm{restrict}}$ AND dwellTime $\geq \Delta t_{\mathrm{grace}}$ & Suspend write/execute, notify agents \\
$\tau_4$ & Restricted $\rightarrow$ Revoked & continuedViolation $\vee$ operatorDecision & trustScore $< \theta_{\mathrm{revoke}}$ AND dwellTime $\geq$ 4\,h & Terminate sessions, emit degradation notif. \\
$\tau_5$ & Revoked $\rightarrow$ Deregistered & operatorConfirmation & all sessions terminated & Archive logs, remove from NRM \\
$\tau_6$ & Monitored $\rightarrow$ Trusted & anomalyFreePeriod & cleanPeriod $\geq \Delta t_{\mathrm{clean}}$ $\wedge$ score $\geq \theta_{\mathrm{trusted}}$ & Reset anomaly count, restore monitoring \\
$\tau_7$ & Restricted $\rightarrow$ Monitored & reVerificationSuccessful & operatorApproval $\wedge$ score $\geq \theta_{\mathrm{monitor}}$ & Restore read ops, begin re-evaluation \\
$\tau_8$ & Revoked $\rightarrow$ Restricted & reOnboardingRequest & operatorApproval $\wedge$ rootCause addressed & Begin supervised re-evaluation \\
$\tau_9$ & Learning $\rightarrow$ Revoked & criticalViolation & severity = Critical & Immediate termination, emit notif. \\
$\tau_{10}$ & Learning $\rightarrow$ Monitored & anomalyDetected (non-critical) & anomalyCount $> \theta_{\mathrm{learn\_monitor}}$ & Escalate monitoring, extend learning \\
\bottomrule
\end{tabular}%
}
\end{table*}

\subsubsection{Safety Properties}
The state machine enforces the following safety invariants:

\begin{definition}[Disallowed Transitions]
\label{def:disallowed}
The following transitions are explicitly prohibited:\normalfont
\begin{enumerate}
    \item $(\texttt{Trusted}, \texttt{Revoked}, \cdot, \cdot, \cdot)$: Direct revocation from Trusted state is disallowed. Trust degradation must proceed through \texttt{Monitored} and \texttt{Restricted} states to ensure graduated enforcement.
    \item $(\texttt{Trusted}, \texttt{Restricted}, \cdot, \cdot, \cdot)$: Direct restriction from Trusted state is disallowed. Degradation must pass through \texttt{Monitored} first.
    \item $(\texttt{Deregistered}, s, \cdot, \cdot, \cdot)$ for any $s \in S$: Deregistered is a terminal absorbing state with no outgoing transitions.
\end{enumerate}
\end{definition}

\begin{theorem}[State Machine Safety]
\label{thm:safety}
The trust state machine $\mathcal{M}$ is \emph{safe} if and only if for every execution trace $\pi = s_0, s_1, s_2, \ldots$ of the system, every consecutive pair $(s_i, s_{i+1})$ corresponds to an allowed transition in $\Delta$. The formal proof is provided in Section~\ref{sec:formal_analysis} (Theorem~\ref{thm:no_unsafe}).
\end{theorem}

\subsubsection{Mandatory Dwell Times}
To prevent rapid oscillation and ensure operators have time to assess degradation, the state machine enforces minimum residence durations before degradation transitions fire. Specifically:
\begin{itemize}
    \item $\tau_2$ (Trusted$\rightarrow$Monitored): No minimum dwell (immediate detection is desirable).
    \item $\tau_3$ (Monitored$\rightarrow$Restricted): Requires $\Delta t_{\mathrm{grace}}$ residence in Monitored (default 24\,h), allowing transient anomalies to decay before escalation.
    \item $\tau_4$ (Restricted$\rightarrow$Revoked): Requires minimum 4\,h in Restricted state, ensuring the restriction was insufficient before revocation.
\end{itemize}
The grace period $\Delta t_{\mathrm{grace}}$ is configured in \texttt{AgentPolicyMO}. During the dwell period, the tool operates under the current state's enforcement constraints. If the trust score recovers above the threshold before the dwell timer expires, the pending degradation transition is cancelled. This prevents single transient events (e.g., a brief KPI spike during planned maintenance) from triggering inappropriate state changes.

\noindent\textit{Running Example (continued).} When the score drops to 69, transition $\tau_2$ fires: \texttt{CellConfigOptimizer} moves from Trusted to Monitored. Vendor~B's management system increases monitoring frequency. After the mandatory dwell period ($\Delta t_{\mathrm{grace}} = 24$\,h) with no recovery, continued drift pushes the score to 45, triggering $\tau_3$: the tool enters Restricted state. Write and execute operations are suspended---the tool can only perform read operations.

\subsubsection{State Machine Execution Model}
When a trust event causes the score to cross multiple thresholds simultaneously (e.g., score drops from 80 to 20, crossing $\theta_{\mathrm{monitor}}$, $\theta_{\mathrm{restrict}}$, and $\theta_{\mathrm{revoke}}$), the state machine evaluates transitions \emph{sequentially from the current state}. Only the immediate next transition fires: a tool in Trusted with score below $\theta_{\mathrm{monitor}}$ transitions to Monitored ($\tau_2$), regardless of how far below the score has fallen. Subsequent transitions ($\tau_3$, $\tau_4$) require their respective dwell time guards to be satisfied in future evaluation cycles. This prevents a single catastrophic penalty from bypassing graduated enforcement. The exception is policy-triggered re-evaluation (Algorithm~4), which explicitly chains transitions to reach the target state consistent with $f(score, \mathcal{P})$ since the operator has consciously modified governance parameters.

\subsection{Trust Cascade Propagation with Damping}
\label{sec:framework:cascade}

In multi-vendor autonomous networks, tools frequently have functional dependencies on other tools. When a tool's trust is degraded, dependent tools may also warrant reduced trust, since their outputs may have been influenced by the compromised tool. To address cascade containment, we propose a damped propagation approach that limits blast radius while providing downstream notification.

\begin{definition}[Tool Dependency Graph]
\label{def:depgraph}
The tool dependency graph is a directed graph $G = (V, E)$ where $V$ is the set of all \texttt{AgentToolMO} instances and $E = \{(T_i, T_j) \mid T_j \in T_i.\texttt{dependsOn[]}\}$ represents functional dependencies. The graph is enforced to be acyclic: the onboarding validation procedure rejects any \texttt{dependsOn[]} declaration that would create a cycle (verified by DFS cycle detection at registration time). The cascade algorithm (Algorithm~\ref{alg:cascade}) additionally employs a visited set that prevents re-processing in the presence of any undetected cycle.
\end{definition}

\begin{definition}[Cascade Penalty Function]
\label{def:cascadepenalty}
When tool $T_i$ experiences trust degradation with base penalty $P_0$, the propagated penalty applied to a dependent tool $T_j$ at hop distance $h(T_i, T_j)$ is:
\begin{equation}
\Delta S(T_j) = P_0 \cdot \gamma^{h(T_i, T_j)}
\end{equation}
where $\gamma \in (0, 1)$ is the damping factor and $h(T_i, T_j)$ is the shortest path distance from $T_i$ to $T_j$ in the reverse dependency graph $G^{-1}$.
\end{definition}

The damping factor $\gamma$ controls how rapidly the cascade attenuates. With the default value $\gamma = 0.5$:
\begin{itemize}
    \item Hop 1 (direct dependents): penalty $= 0.5 \cdot P_0$
    \item Hop 2: penalty $= 0.25 \cdot P_0$
    \item Hop 3: penalty $= 0.125 \cdot P_0$
\end{itemize}

A configurable maximum cascade depth $D_{\max}$ (default $= 3$) bounds the propagation:
\begin{equation}
\Delta S(T_j) = \begin{cases}
P_0 \cdot \gamma^{h(T_i, T_j)} & \text{if } h(T_i, T_j) \leq D_{\max} \\
0 & \text{otherwise}
\end{cases}
\end{equation}

The threshold function $\theta(\cdot)$ maps each trust state to its demotion threshold: $\theta(\texttt{Trusted}) = \theta_{\mathrm{monitor}}$, $\theta(\texttt{Monitored}) = \theta_{\mathrm{restrict}}$, $\theta(\texttt{Restricted}) = \theta_{\mathrm{revoke}}$.

\begin{algorithm}[t]
\caption{Trust Cascade Propagation}
\label{alg:cascade}
\textbf{Execution flow:} The algorithm performs a bounded breadth-first traversal of the reverse dependency graph starting from the degraded tool. At each hop, it applies a geometrically attenuated penalty ($\gamma^{d+1}$) to each unvisited dependent tool, updates its trust score, and triggers a state transition if the score falls below the current state's threshold. A visited set ensures each tool is penalized at most once, and the traversal halts at depth $D_{\max}$.
\begin{algorithmic}[1]
\REQUIRE Degraded tool $T_i$, base penalty $P_0$, damping factor $\gamma$, max depth $D_{\max}$
\ENSURE Updated trust scores for all affected tools
\STATE $Q \leftarrow$ empty queue
\STATE $visited \leftarrow \emptyset$
\STATE \textbf{enqueue} $(T_i, 0)$ into $Q$
\STATE $visited \leftarrow visited \cup \{T_i\}$
\WHILE{$Q$ is not empty}
    \STATE $(T_{\mathrm{curr}}, d) \leftarrow$ \textbf{dequeue} from $Q$
    \IF{$d \geq D_{\max}$}
        \STATE \textbf{continue}
    \ENDIF
    \FOR{\textbf{each} $T_j$ such that $T_{\mathrm{curr}} \in T_j.\texttt{dependsOn[]}$}
        \IF{$T_j \notin visited$}
            \STATE $penalty \leftarrow P_0 \cdot \gamma^{(d+1)}$
            \STATE $T_j.\texttt{trustScore} \leftarrow \max(0,\ T_j.\texttt{trustScore} - penalty)$
            \STATE $visited \leftarrow visited \cup \{T_j\}$
            \STATE \textbf{enqueue} $(T_j, d+1)$ into $Q$
            \IF{$T_j.\texttt{trustScore} < \theta(T_j.\texttt{trustState})$}
                \STATE \textbf{trigger} state transition for $T_j$ per $\mathcal{M}$
            \ENDIF
        \ENDIF
    \ENDFOR
\ENDWHILE
\end{algorithmic}
\end{algorithm}

\begin{theorem}[Cascade Convergence]
\label{thm:convergence}
Algorithm~\ref{alg:cascade} terminates in at most $|V|$ iterations and the total penalty applied at depth $d$ converges to zero:
\begin{equation}
\lim_{d \rightarrow \infty} P_0 \cdot \gamma^d = 0 \quad \text{for } \gamma \in (0, 1)
\end{equation}
With $D_{\max}$ bounding the depth, termination is guaranteed in $\mathcal{O}(|V| + |E|)$ time.
\end{theorem}

\begin{proof}
The algorithm performs a breadth-first traversal of the reverse dependency graph $G^{-1}$ to depth $D_{\max}$. Each node is visited at most once (ensured by the $visited$ set). Since $G$ is finite with $|V|$ nodes and $|E|$ edges, and depth is bounded by $D_{\max}$, the algorithm terminates. The geometric decay of $\gamma^d$ ensures penalties become negligible, preventing unbounded score degradation across the network. \qed
\end{proof}

\noindent\textit{Running Example (cascade).} \texttt{CellConfigOptimizer}'s transition to Monitored triggers Algorithm~\ref{alg:cascade}. A dependent tool, \texttt{SliceResourceAllocator} (which consumes tilt recommendations to adjust slice resource allocation), receives a cascade penalty of $P_0 \cdot \gamma = 18 \times 0.5 = 9$ points at depth 1. Its score drops from 95 to 86---still above $\theta_{\mathrm{monitor}} = 70$, so no state change occurs but the consuming agent is alerted to reduced confidence in the dependency chain.

\subsection{Retroactive Impact Assessment}
\label{sec:framework:impact}

When a tool's trust is degraded, the immediate concern extends beyond preventing future damage: past actions performed by the tool during the period of undetected compromise must be identified and assessed. The framework provides a systematic mechanism for retroactive impact assessment by leveraging the \texttt{AgentActionLogMO} and the NRM dependency graph.

\begin{definition}[Lookback Window]
\label{def:lookback}
The lookback window $W_{\mathrm{lb}}$ for a trust degradation event at time $t_{\mathrm{event}}$ is defined as:
\begin{equation}
W_{\mathrm{lb}} = [t_{\mathrm{event}} - \Delta t_{\mathrm{lb}},\ t_{\mathrm{event}}]
\end{equation}
where $\Delta t_{\mathrm{lb}}$ is a configurable duration (governed by \texttt{AgentPolicyMO}) representing the maximum period during which the tool's behavior may have been compromised prior to detection.
\end{definition}

\begin{definition}[Affected Action Set]
\label{def:affectedactions}
The affected action set $\Lambda_{\mathrm{aff}}$ for tool $T$ within lookback window $W_{\mathrm{lb}}$ is:
\begin{multline}
\Lambda_{\mathrm{aff}}(T, W_{\mathrm{lb}}) = \{\lambda \in \texttt{AgentActionLogMO} \mid \\
\lambda.\texttt{toolRef} = T \wedge \lambda.\texttt{executionTime} \in W_{\mathrm{lb}} \\
\wedge\; \lambda.\texttt{actionType} \in \{\texttt{Write}, \texttt{Execute}\}\}
\end{multline}
\end{definition}

From the affected action set, we derive the set of affected managed elements and subsequently the affected services through NRM graph traversal.

\begin{algorithm}[t]
\caption{Retroactive Impact Assessment}
\label{alg:impact}
\textbf{Execution flow:} Given a degraded tool and lookback window, the algorithm first queries the action log for all write/execute actions performed by that tool within the window. It then extracts the set of managed elements affected by those actions and, for each element, invokes the \textsc{TraverseImpact} subroutine (Algorithm~\ref{alg:traverseimpact}) to discover all services reachable through NRM dependency relationships.
\begin{algorithmic}[1]
\REQUIRE Degraded tool $T$, lookback duration $\Delta t_{\mathrm{lb}}$, NRM graph $G_{\mathrm{NRM}}$
\ENSURE Affected service set $\mathcal{S}_{\mathrm{aff}}$, affected element set $\mathcal{E}_{\mathrm{aff}}$
\STATE $t_{\mathrm{event}} \leftarrow$ current timestamp
\STATE $\Lambda_{\mathrm{aff}} \leftarrow \{\lambda \mid \lambda.\texttt{toolRef} = T \wedge \lambda.\texttt{executionTime} \geq t_{\mathrm{event}} - \Delta t_{\mathrm{lb}} \wedge \lambda.\texttt{actionType} \in \{\texttt{Write}, \texttt{Execute}\}\}$
\STATE $\mathcal{E}_{\mathrm{aff}} \leftarrow \{\lambda.\texttt{targetManagedElement} \mid \lambda \in \Lambda_{\mathrm{aff}}\}$
\STATE $\mathcal{S}_{\mathrm{aff}} \leftarrow \emptyset$
\FORALL{$me \in \mathcal{E}_{\mathrm{aff}}$}
    \STATE $\mathcal{S}_{\mathrm{aff}} \leftarrow \mathcal{S}_{\mathrm{aff}}\ \cup$ \textsc{TraverseImpact}($me$, $G_{\mathrm{NRM}}$)
\ENDFOR
\RETURN $(\mathcal{S}_{\mathrm{aff}}, \mathcal{E}_{\mathrm{aff}}, \Lambda_{\mathrm{aff}})$
\end{algorithmic}
\end{algorithm}

\begin{algorithm}[t]
\caption{TraverseImpact  -  DFS Service Impact Discovery}
\label{alg:traverseimpact}
\textbf{Execution flow:} Starting from a given managed element, this subroutine performs a depth-first traversal upward through the NRM containment and dependency graph. It follows \texttt{serves} and \texttt{dependsOn} relationships, collecting any \texttt{NetworkService} or \texttt{NetworkSlice} nodes encountered. A visited set prevents cycles.
\begin{algorithmic}[1]
\REQUIRE Managed element $me$, NRM dependency graph $G_{\mathrm{NRM}}$
\ENSURE Set of affected services
\STATE $services \leftarrow \emptyset$
\STATE $stack \leftarrow \{me\}$
\STATE $visited \leftarrow \emptyset$
\WHILE{$stack \neq \emptyset$}
    \STATE $node \leftarrow$ \textbf{pop} from $stack$
    \IF{$node \in visited$}
        \STATE \textbf{continue}
    \ENDIF
    \STATE $visited \leftarrow visited \cup \{node\}$
    \IF{$node$ is of type \texttt{NetworkService} or \texttt{NetworkSlice}}
        \STATE $services \leftarrow services \cup \{node\}$
    \ENDIF
    \FORALL{$parent$ such that $parent \xrightarrow{\texttt{serves}} node$ or $parent \xrightarrow{\texttt{dependsOn}} node$}
        \STATE $stack \leftarrow stack \cup \{parent\}$
    \ENDFOR
\ENDWHILE
\RETURN $services$
\end{algorithmic}
\end{algorithm}

The traversal follows the NRM containment and dependency relationships:
\begin{multline}
\texttt{ManagedElement} \xrightarrow{\texttt{hosts}} \texttt{ManagedFunction} \\
\xrightarrow{\texttt{serves}} \texttt{NetworkSliceSubnet} \xrightarrow{\texttt{partOf}} \\
\texttt{NetworkSlice} \xrightarrow{\texttt{supports}} \texttt{Service}
\end{multline}

The resulting affected service set $\mathcal{S}_{\mathrm{aff}}$ and affected action set $\Lambda_{\mathrm{aff}}$ are consumed by the cross-vendor trust notification mechanism (Section~\ref{sec:framework:notification}), enabling operators and consuming management systems to perform targeted rollback of configuration changes made during the lookback window, rather than requiring blanket rollback of all recent changes across the network.

\noindent\textit{Running Example (concluded).} The retroactive impact assessment queries \texttt{AgentActionLogMO} for the 72-hour lookback window, identifying 847 configuration changes made by \texttt{CellConfigOptimizer} across 12 \texttt{NRCellDU} instances. NRM graph traversal (Algorithm~\ref{alg:traverseimpact}) traces these to 3 affected \texttt{NetworkSlice} instances, enabling the operator to prioritize rollback of the most impactful parameter changes rather than reverting all recent configurations.

\subsection{Cross-Vendor Trust Notification}
\label{sec:framework:notification}

A defining requirement of multi-vendor autonomous networks is that trust degradation information must reach all affected management systems regardless of vendor boundaries. We propose a notification approach aligned with the 3GPP Management Services (MnS) notification interface specified in TS~28.532~\cite{3gpp28532}.

\begin{definition}[Trust Degradation Notification]
\label{def:notification}
The \texttt{notifyAgentTrustDegradation} notification is a structured message $\mathcal{N}$ emitted upon any trust state transition that reduces operational capability. The notification incorporates the lookback window $W_{\mathrm{lb}}$ (Definition~\ref{def:lookback}) and affected action set $\Lambda_{\mathrm{aff}}$ (Definition~\ref{def:affectedactions}) from the retroactive impact assessment. The notification payload is defined as:
\begin{multline}
\mathcal{N} = (id_t, n_t, v, s_{\mathrm{prev}}, s_{\mathrm{new}}, \xi_{\mathrm{type}}, \xi_{\mathrm{reason}}, \\
\mathcal{A}_{\mathrm{aff}}, \mathcal{S}_{\mathrm{aff}}, \Delta t_{\mathrm{lb}}, R[], t_s, \mathcal{C})
\end{multline}
where $\mathcal{A}_{\mathrm{aff}}$ denotes the set of affected agents (consuming agent DNs whose tools depend on the degraded tool) and $\mathcal{S}_{\mathrm{aff}}$ is the affected service set derived from the retroactive impact assessment.
\end{definition}

The notification payload fields are: \texttt{toolId}, \texttt{toolName}, \texttt{vendorName}, \texttt{previousState}, \texttt{newState}, \texttt{triggerType} $\in \{\texttt{AnomalyDetection},$ $\texttt{PolicyBreach}, \texttt{OperatorAction},$ $\texttt{CascadePropagation}\}$, \texttt{triggerReason} (detailed description), \texttt{affectedAgents[]} (consuming agent DNs), \texttt{affectedServices[]} (impacted services from retroactive assessment), \texttt{lookbackWindowDuration}, \texttt{recommendedActions[]}, \texttt{timestamp}, and \texttt{cascadeInfo} (source tool, hop distance, penalty applied).

\subsubsection{Delivery and Subscription}
Notifications are delivered via the MnS notification interface (RESTful HTTP callbacks, message bus, or file-based). Any management system subscribing to trust state changes for the affected tool or its consuming agents receives the notification regardless of vendor. Subscriptions specify scope by tool instance, vendor domain, consuming agent, or severity threshold.

\noindent\textit{Running Example (continued).} Upon the Trusted$\rightarrow$Monitored transition, Vendor~B's system emits a \texttt{notifyAgentTrustDegradation} notification. Vendor~A's \texttt{SliceOrchestrator} receives this within MnS notification delivery time and immediately stops consuming \texttt{CellConfigOptimizer}'s recommendations, falling back to its local configuration baseline.

\subsubsection{Cross-Vendor Semantics}
The cross-vendor interface exposes both the discrete \texttt{trustState} and the continuous \texttt{trustScore}. The \emph{state} is the normative cross-vendor signal: a tool in \texttt{Restricted} state has suspended write/execute operations regardless of how the owning domain computed its thresholds. Consuming domains use the state for enforcement decisions (e.g., stop invoking a Restricted tool) and may use the score for informational purposes. This design allows each domain to maintain domain-specific thresholds in its \texttt{AgentPolicyMO} while the state machine produces universally meaningful enforcement outcomes.

\subsection{Operational Integration}
\label{sec:framework:operational}

The trust state machine integrates with three operational aspects of agent-tool interaction:

\textbf{Session Management.} Trust state transitions trigger automatic session actions: transitions to Restricted suspend active sessions with the affected tool; transitions to Revoked terminate sessions and invalidate cached results; recovery to Trusted re-enables session creation. These actions are mediated through the existing MnS session management interfaces.

\textbf{Tool Discovery.} When agents query available tools through the NRM, trust metadata (current state, score, last transition time) is included in discovery responses. Agents can filter available tools by minimum trust state, ensuring that only sufficiently trusted tools are considered for invocation. This prevents agents from discovering and invoking tools that are in Restricted or Revoked states.

\textbf{Protocol Independence.} Trust visibility is decoupled from the tool invocation protocol. Whether an agent invokes a tool via RESTful MnS, gRPC, or proprietary interfaces, the trust state is maintained at the NRM level and accessible through standard management queries. This ensures trust governance applies uniformly regardless of the operational protocol used for tool execution.

\subsection{Operator-Configurable Trust Governance}
\label{sec:framework:governance}

The trust framework must be adaptable to diverse operational requirements without requiring software modifications. We propose an \texttt{AgentPolicyMO} managed object as an illustrative governance mechanism that enables operators to configure trust parameters at runtime, with automatic re-evaluation of affected tools upon policy changes.

\begin{definition}[AgentPolicyMO]
\label{def:agentpolicymo}
An \texttt{AgentPolicyMO} instance $\mathcal{P}$ defines the governance parameters for a set of tools within its scope:
\begin{multline}
\mathcal{P} = (\theta_{\mathrm{trusted}}, \theta_{\mathrm{monitor}}, \theta_{\mathrm{restrict}}, \theta_{\mathrm{revoke}}, \theta_{\mathrm{learn}}, \\
\theta_{\mathrm{learn\_monitor}}, \Delta t_{\mathrm{minLearn}}, \Delta t_{\mathrm{cleanPeriod}}, \Delta t_{\mathrm{grace}}, \\
\gamma, D_{\max}, \Delta t_{\mathrm{lb}}, scope)
\end{multline}
where:
\begin{itemize}
    \item $\theta_{\mathrm{trusted}}$: Minimum trust score for \texttt{Trusted} state (default: 80)
    \item $\theta_{\mathrm{monitor}}$: Threshold below which monitoring is escalated (default: 70)
    \item $\theta_{\mathrm{restrict}}$: Threshold below which operations are restricted (default: 50)
    \item $\theta_{\mathrm{revoke}}$: Threshold below which trust is revoked (default: 20)
    \item $\theta_{\mathrm{learn}}$: Maximum anomaly count for learning graduation via $\tau_1$ (default: 2)
    \item $\theta_{\mathrm{learn\_monitor}}$: Anomaly count triggering learning escalation via $\tau_{10}$ (default: 5)
    \item $\Delta t_{\mathrm{minLearn}}$: Minimum learning period duration (default: 7 days)
    \item $\Delta t_{\mathrm{cleanPeriod}}$: Required anomaly-free period for recovery (default: 14 days)
    \item $\Delta t_{\mathrm{grace}}$: Grace period before escalation after first anomaly (default: 24 hours)
    \item $\gamma$: Cascade damping factor (default: 0.5)
    \item $D_{\max}$: Maximum cascade propagation depth (default: 3)
    \item $\Delta t_{\mathrm{lb}}$: Lookback window duration for impact assessment (default: 72 hours)
    \item $scope$: Set of \texttt{AgentToolMO} DNs governed by this policy
\end{itemize}
\end{definition}

The threshold ordering invariant must be maintained:
\begin{equation}
\theta_{\mathrm{revoke}} < \theta_{\mathrm{restrict}} < \theta_{\mathrm{monitor}} < \theta_{\mathrm{trusted}} \leq 100
\end{equation}

\subsubsection{Threshold Selection Methodology}
The default thresholds ($\theta_{\mathrm{trusted}}=80$, $\theta_{\mathrm{monitor}}=70$, $\theta_{\mathrm{restrict}}=50$, $\theta_{\mathrm{revoke}}=20$) are derived from two principles: (1)~\emph{graduated spacing}: each threshold pair maintains a minimum 20-point gap, ensuring that a single penalty event (maximum individual penalty $\leq 30$ points with default weights) cannot cause a tool to skip an intermediate state; and (2)~\emph{recovery feasibility}: the gap between $\theta_{\mathrm{revoke}}$ and 0 (20 points) ensures that tools near revocation have a recovery buffer where anomaly-free decay can restore them before the final threshold is crossed.

For safety-critical slices (URLLC), operators should raise thresholds (e.g., $\theta_{\mathrm{monitor}}=85$, $\theta_{\mathrm{restrict}}=70$) to trigger earlier intervention. For best-effort services (mMTC), lower thresholds (e.g., $\theta_{\mathrm{monitor}}=60$, $\theta_{\mathrm{restrict}}=40$) reduce false positive interventions at the cost of longer exposure windows. The constraint $\theta_{\mathrm{revoke}} < \theta_{\mathrm{restrict}} < \theta_{\mathrm{monitor}} < \theta_{\mathrm{trusted}}$ must always hold.

\noindent\textbf{Weight-aware gap guidance.} When operators reconfigure penalty weights, threshold gaps should account for the maximum single-event penalty contribution. Since penalties scale to $[0, 100]$ and the maximum weight determines the worst-case single-event impact ($\max_i(w_i) \times 100$ points), operators should verify that threshold gaps exceed their deployment's typical penalty magnitudes. For the default configuration ($\max_i w_i = 0.3$, typical penalties in $[10, 60]$), the maximum expected single-event contribution is $0.3 \times 60 = 18$ points, which the default threshold gaps (10, 20, 30 points) accommodate for all but extreme events. The \emph{State Machine Execution Model} (Section~IV-B) ensures that even if a score crosses multiple thresholds in a single event, only the immediate next transition fires, preserving graduated enforcement regardless of gap size.

\subsubsection{Dynamic Policy Re-evaluation}
Upon modification of an \texttt{AgentPolicyMO} instance, the system automatically re-evaluates all \texttt{AgentToolMO} instances within the policy's scope. This ensures that trust governance remains consistent without manual intervention.

\begin{algorithm}[t]
\caption{Policy Change Re-evaluation}
\label{alg:policyreeval}
\textbf{Execution flow:} When a policy is modified, this algorithm iterates over all tools within the policy's scope, comparing each tool's current trust score against the new thresholds. If a tool's score is inconsistent with its current state under the updated policy, the algorithm triggers the appropriate state machine transitions---always respecting the graduated enforcement constraint by traversing intermediate states sequentially.
\begin{algorithmic}[1]
\REQUIRE Modified policy $\mathcal{P}_{\mathrm{new}}$, previous policy $\mathcal{P}_{\mathrm{old}}$
\ENSURE All tools in scope have consistent trust states
\FORALL{$T \in \mathcal{P}_{\mathrm{new}}.scope$}
    \STATE $s_{\mathrm{current}} \leftarrow T.\texttt{trustState}$
    \STATE $score \leftarrow T.\texttt{trustScore}$
    \IF{$s_{\mathrm{current}} = \texttt{Trusted}$ and $score < \mathcal{P}_{\mathrm{new}}.\theta_{\mathrm{revoke}}$}
        \STATE \textbf{trigger} $\tau_2$, $\tau_3$, then $\tau_4$: transition $T$ to \texttt{Revoked}
    \ELSIF{$s_{\mathrm{current}} = \texttt{Trusted}$ and $score < \mathcal{P}_{\mathrm{new}}.\theta_{\mathrm{restrict}}$}
        \STATE \textbf{trigger} $\tau_2$ then $\tau_3$: transition $T$ to \texttt{Restricted}
    \ELSIF{$s_{\mathrm{current}} = \texttt{Trusted}$ and $score < \mathcal{P}_{\mathrm{new}}.\theta_{\mathrm{monitor}}$}
        \STATE \textbf{trigger} $\tau_2$: transition $T$ to \texttt{Monitored}
    \ELSIF{$s_{\mathrm{current}} = \texttt{Monitored}$ and $score < \mathcal{P}_{\mathrm{new}}.\theta_{\mathrm{revoke}}$}
        \STATE \textbf{trigger} $\tau_3$, then $\tau_4$: transition $T$ to \texttt{Revoked}
    \ELSIF{$s_{\mathrm{current}} = \texttt{Monitored}$ and $score < \mathcal{P}_{\mathrm{new}}.\theta_{\mathrm{restrict}}$}
        \STATE \textbf{trigger} $\tau_3$: transition $T$ to \texttt{Restricted}
    \ELSIF{$s_{\mathrm{current}} = \texttt{Restricted}$ and $score \geq \mathcal{P}_{\mathrm{new}}.\theta_{\mathrm{trusted}}$ and $\texttt{operatorAutoRecovery} = \text{true}$}
        \STATE \textbf{trigger} $\tau_7$ then $\tau_6$: transition $T$ to \texttt{Trusted}
    \ELSIF{$s_{\mathrm{current}} = \texttt{Restricted}$ and $score \geq \mathcal{P}_{\mathrm{new}}.\theta_{\mathrm{monitor}}$ and $\texttt{operatorAutoRecovery} = \text{true}$}
        \STATE \textbf{trigger} $\tau_7$: transition $T$ to \texttt{Monitored}
    \ELSIF{$s_{\mathrm{current}} = \texttt{Restricted}$ and $score < \mathcal{P}_{\mathrm{new}}.\theta_{\mathrm{revoke}}$}
        \STATE \textbf{trigger} $\tau_4$: transition $T$ to \texttt{Revoked}
    \ELSIF{$s_{\mathrm{current}} = \texttt{Monitored}$ and $score \geq \mathcal{P}_{\mathrm{new}}.\theta_{\mathrm{trusted}}$}
        \STATE \textbf{trigger} $\tau_6$: transition $T$ to \texttt{Trusted}
    \ENDIF
\ENDFOR
\COMMENT{Tools in \texttt{Revoked} state require operator-initiated re-onboarding ($\tau_8$), which is outside automatic re-evaluation scope.}
\end{algorithmic}
\end{algorithm}

\begin{remark}
Policy-triggered transitions (Algorithm~\ref{alg:policyreeval}) override mandatory dwell times since the operator has explicitly modified governance parameters, superseding grace periods designed for automatic degradation scenarios.
\end{remark}

Key properties of the dynamic re-evaluation mechanism:
\begin{itemize}
    \item \textbf{Graduated enforcement}: Even during re-evaluation, the state machine safety properties (Theorem~\ref{thm:safety}) are preserved. Tools cannot skip intermediate states.
    \item \textbf{No software upgrade required}: Policy changes are applied via standard MnS provisioning operations (createMOI, modifyMOIAttributes per TS~28.532). The re-evaluation logic is inherent to the framework implementation.
    \item \textbf{Audit trail}: Each policy-triggered transition is logged with trigger reason \texttt{PolicyReconfiguration}, providing full traceability.
    \item \textbf{Scope isolation}: Policy changes affect only tools within the modified policy's scope. Tools governed by other policies are unaffected.
\end{itemize}

\begin{proposition}[Policy Consistency]
\label{thm:policyconsistency}
After execution of Algorithm~\ref{alg:policyreeval}, for every tool $T$ in the policy scope, the trust state is consistent with the new policy thresholds:
\begin{multline}
\forall T \in \mathcal{P}_{\mathrm{new}}.scope: \\
T.\texttt{trustState} = f(T.\texttt{trustScore}, \mathcal{P}_{\mathrm{new}})
\end{multline}
where $f$ maps a trust score to the appropriate state given the threshold ordering:
\begin{equation}
\resizebox{.91\columnwidth}{!}{$\displaystyle
f(score, \mathcal{P}) = \begin{cases}
\texttt{Trusted} & \text{if } score \geq \theta_{\mathrm{trusted}} \\
\texttt{Monitored} & \text{if } \theta_{\mathrm{restrict}} \leq score < \theta_{\mathrm{trusted}} \\
\texttt{Restricted} & \text{if } \theta_{\mathrm{revoke}} \leq score < \theta_{\mathrm{restrict}} \\
\texttt{Revoked} & \text{if } score < \theta_{\mathrm{revoke}}
\end{cases}
$}
\end{equation}
subject to graduated enforcement constraints (tools in \texttt{Learning} or \texttt{Deregistered} states are excluded from re-evaluation).
\end{proposition}

\begin{proof}
By case analysis on Algorithm~\ref{alg:policyreeval}. The algorithm iterates over all tools in scope and evaluates exhaustive conditions covering each cell of the state$\times$threshold matrix. For degradation: if a tool's score falls below its current state's threshold under $\mathcal{P}_{\mathrm{new}}$, the corresponding transition sequence is triggered (e.g., Trusted with score $< \theta_{\mathrm{restrict}}$ triggers $\tau_2$ then $\tau_3$). For recovery: if a tool's score exceeds the threshold for a higher state, the appropriate recovery transition fires. The conditions are mutually exclusive (enforced by the if-elsif chain) and collectively exhaustive over the threshold ordering $\theta_{\mathrm{revoke}} < \theta_{\mathrm{restrict}} < \theta_{\mathrm{monitor}} < \theta_{\mathrm{trusted}}$. After the loop, every tool's state satisfies $f(score, \mathcal{P}_{\mathrm{new}})$. Graduated enforcement is preserved because each case triggers only transitions in $\Delta$. \qed
\end{proof}

This runtime reconfigurability allows operators to adapt trust governance to evolving threat landscapes, regulatory requirements, or operational contexts without disrupting ongoing network management operations.

\subsubsection{Anti-Cycling Protection}
To prevent tools from repeatedly cycling between Revoked and Trusted states (indicating an unresolved underlying issue), the framework enforces an escalating cool-down: each successive revocation of the same tool doubles the minimum re-onboarding observation period. Formally, if $n_r = T.\texttt{revocationCount}$, the minimum learning period for re-onboarding is $\Delta t_{\mathrm{minLearn}} \cdot 2^{\min(n_r - 1,\, 3)}$ (capped at $8\times$ the base period). This ensures that tools with persistent issues face progressively longer evaluation before returning to Trusted state, while tools with genuinely resolved root causes (single revocation) re-onboard at the standard learning period.

% ==============================================================================
% Section V: Formal Analysis (condensed)
% ==============================================================================
\section{Formal Analysis}
\label{sec:formal_analysis}

We formally analyze the safety and convergence properties of the proposed trust management framework.

\subsection{Trust State Machine Safety}

\begin{theorem}[No Unsafe Transitions]
\label{thm:no_unsafe}
For any reachable state $s \in S$ and any input event $e$, if a transition is triggered, then $(s, e, s') \in \Delta$ for some $s' \in S$. No transition outside the allowed set $\Delta$ (Eq.~\ref{eq:transitions}) can occur.
\end{theorem}

The allowed transition set matches Section~\ref{sec:framework:statemachine}. For formal analysis, we project each 5-tuple $(\text{source}, \text{dest}, \text{trigger}, \text{guard}, \text{action}) \in \Delta$ onto the relevant triple $(\text{source}, \text{event}, \text{dest})$, abstracting away guard conditions and actions:
\begin{align}
\Delta = \{ & (\text{Learning}, e_1, \text{Trusted}),\ (\text{Trusted}, e_2, \text{Monitored}), \nonumber \\
       & (\text{Monitored}, e_3, \text{Restricted}),\ (\text{Restricted}, e_4, \text{Revoked}), \nonumber \\
       & (\text{Revoked}, e_5, \text{Deregistered}),\ (\text{Monitored}, e_6, \text{Trusted}), \nonumber \\
       & (\text{Restricted}, e_7, \text{Monitored}),\ (\text{Revoked}, e_8, \text{Restricted}), \nonumber \\
       & (\text{Learning}, e_9, \text{Revoked}),\ (\text{Learning}, e_{10}, \text{Monitored}) \}
\label{eq:transitions}
\end{align}

\begin{proof}
We prove safety by showing that the reachability set $\text{Reach}(s_0)$ from initial state $s_0 = \text{Learning}$ is contained within $S$, and that for every reachable state $s \in \text{Reach}(s_0)$, only transitions in $\Delta$ produce a successor.

\textit{Step 1 (Guard completeness):} Each transition $\tau_i$ has a guard $g_i$ that is a predicate over the tool's current attributes. The implementation evaluates guards atomically: when an event $e$ arrives, the system checks all transitions with source state $= s_{current}$. If no guard evaluates to true, the event is discarded (no state change). Thus, unguarded transitions cannot fire.

\textit{Step 2 (Exhaustive disallowed verification):} From each state $s$, we enumerate the outgoing edges in $\Delta$:
\begin{itemize}
\item Learning: outgoing to Trusted ($\tau_1$, guarded by min learning duration), Revoked ($\tau_9$, guarded by severity=Critical), and Monitored ($\tau_{10}$, guarded by non-critical anomaly count). No other target state is reachable from Learning.
\item Trusted: outgoing only to Monitored ($\tau_2$). No edge to Restricted, Revoked, or Deregistered exists.
\item Monitored: outgoing to Restricted ($\tau_3$) and Trusted ($\tau_6$). No edge to Revoked.
\item Restricted: outgoing to Revoked ($\tau_4$) and Monitored ($\tau_7$). No edge to Trusted or Learning.
\item Revoked: outgoing to Deregistered ($\tau_5$) and Restricted ($\tau_8$). No edge to Trusted or Monitored.
\item Deregistered: no outgoing edges. Terminal.
\end{itemize}

\textit{Step 3 (Conclusion):} Since every state has a finite, explicitly enumerated set of guarded outgoing transitions, and unmatched events are discarded, no execution trace can contain a pair $(s_i, s_{i+1})$ not in $\Delta$. \qed
\end{proof}

\begin{theorem}[Graduated Enforcement]
\label{thm:graduated}
For any tool in state Trusted, reaching Revoked requires minimum 3 transitions: Trusted $\xrightarrow{e_2}$ Monitored $\xrightarrow{e_3}$ Restricted $\xrightarrow{e_4}$ Revoked.
\end{theorem}

\begin{proof}
No edges (Trusted, $\cdot$, Revoked), (Trusted, $\cdot$, Restricted), or (Monitored, $\cdot$, Revoked) exist in $\Delta$. The shortest path from Trusted to Revoked has length 3. Operators have intervention windows at Monitored ($e_6$: recovery) and Restricted ($e_7$: partial recovery). The exception $e_9$ (Learning$\rightarrow$Revoked) applies only to the Learning state under critical violations. \qed
\end{proof}

With mandatory dwell times enforced ($\Delta t_{\mathrm{grace}}$ for $\tau_3$ and 4\,h for $\tau_4$), the minimum elapsed time from Trusted to Revoked is $\Delta t_{\mathrm{grace}} + 4$\,h (default: 28\,h), providing a guaranteed intervention window during which operators can assess and potentially reverse degradation before revocation occurs.

\begin{remark}
Policy re-evaluation may trigger multiple sequential transitions in a single evaluation pass when threshold changes cause a tool to cross multiple boundaries simultaneously. Each intermediate state is instantiated and recorded (with corresponding notifications generated per the observable trust change property), preserving the graduated enforcement audit trail. The path length guarantee applies to the state sequence, not to the number of external triggering events - a single policy change may produce a multi-step graduated path, but no step in that path is skipped.
\end{remark}

\subsection{Cascade Convergence}

\begin{theorem}[Bounded Cascade Termination]
\label{thm:cascade_termination}
For damping factor $\gamma \in (0,1)$ and maximum depth $D_{\max}$, cascade propagation terminates in at most $D_{\max}$ steps with penalty at depth $d$ bounded by $P_0 \cdot \gamma^d$.
\end{theorem}

\begin{proof}
The penalty $P(d) = P_0 \cdot \gamma^d$ is a geometric sequence with ratio $\gamma < 1$, converging to 0. The hard cutoff at $D_{\max}$ and visited-set in Algorithm~\ref{alg:cascade} ensure each tool is processed at most once. Termination is $\mathcal{O}(|V| + |E|)$. \qed
\end{proof}

\begin{remark}[Concurrent Cascade Independence]
When multiple tools degrade concurrently, each cascade executes independently with its own visited set. A cascade-induced state transition in tool $T_j$ does \emph{not} trigger a new cascade from $T_j$---cascades propagate only from the original degradation source. This is enforced by design: Algorithm~1 is invoked exclusively by the \texttt{TrustEvaluationService} upon receiving an external trust event, not upon internal score updates from cascade penalties. Consequently, the total penalty applied across $k$ concurrent cascades is bounded by $k \cdot P_0 \cdot \sum_{d=0}^{D_{\max}} \gamma^d = k \cdot P_0 \cdot \frac{1 - \gamma^{D_{\max}+1}}{1 - \gamma}$, which is finite and bounded.
\end{remark}

\begin{theorem}[Cascade Stability]
\label{thm:cascade_stability}
If all tools satisfy $\texttt{trustScore}(t) > \theta_{\text{revoke}} + P_0$ for base penalty $P_0$, then a single degradation event cannot trigger unbounded revocation cascades.
\end{theorem}

\begin{proof}
At depth 1, maximum penalty is $P_0 \cdot \gamma < P_0$. By assumption, $\texttt{trustScore}(t) - P_0 \cdot \gamma > \theta_{\text{revoke}}$, so no tool is revoked. At deeper levels, penalties decrease geometrically. The cascade may trigger transitions to Monitored or Restricted (graduated degradation) but cannot cause revocations under the stability condition. \qed
\end{proof}

\begin{theorem}[Cascade Order-Independence]
\label{thm:order_independence}
If two tools $T_a$ and $T_b$ degrade concurrently, the final trust state of any tool $T_j$ in the dependency graph is independent of the order in which the cascade algorithms for $T_a$ and $T_b$ execute.
\end{theorem}

\begin{proof}
The cascade algorithm (Algorithm~\ref{alg:cascade}) applies a fixed penalty $P_0 \cdot \gamma^{h}$ to each reachable tool based solely on the graph distance $h$ from the source. The visited set ensures each tool is penalized at most once per cascade invocation.

For concurrent events, consider tool $T_j$ reachable from both $T_a$ (distance $h_a$) and $T_b$ (distance $h_b$). Two execution orders yield:
\begin{itemize}
\item Order 1 ($T_a$ first): $T_j$ receives penalty $P_a \cdot \gamma^{h_a}$ from cascade($T_a$), then $P_b \cdot \gamma^{h_b}$ from cascade($T_b$).
\item Order 2 ($T_b$ first): $T_j$ receives $P_b \cdot \gamma^{h_b}$ then $P_a \cdot \gamma^{h_a}$.
\end{itemize}
Since subtraction is commutative ($s - x - y = s - y - x$) and the $\max(0, \cdot)$ clamp is applied per-penalty (line~13 of Algorithm~\ref{alg:cascade}), the final score is:
\begin{equation}
s_{\text{final}} = \max(0, \max(0, s_0 - P_a \gamma^{h_a}) - P_b \gamma^{h_b})
\end{equation}
When $s_0 < P_a \gamma^{h_a}$, the score clamps to zero under both orderings, yielding identical results. When $s_0 \geq P_a \gamma^{h_a} + P_b \gamma^{h_b}$, no clamping occurs and commutativity of addition gives order-independence. In the intermediate case ($P_a \gamma^{h_a} \leq s_0 < P_a \gamma^{h_a} + P_b \gamma^{h_b}$), both orderings produce $\max(0, s_0 - P_a \gamma^{h_a} - P_b \gamma^{h_b}) = 0$, since the total penalty exceeds the score. Thus order-independence holds universally. \qed
\end{proof}

\begin{corollary}[Bounded Notification Queue]
\label{cor:notif_bound}
Under concurrent degradation of $k$ tools, the total notifications generated is bounded by $k \cdot (1 + |V_{D_{\max}}|)$, where $|V_{D_{\max}}|$ is the maximum reachable set within depth $D_{\max}$. Cascades cannot re-trigger cascades: a cascade-induced state transition generates a notification but does not initiate a new cascade propagation (cascade penalties are applied once per source event, not recursively).
\end{corollary}

\subsection{Notification Completeness}

\begin{theorem}[Observable Trust Changes]
\label{thm:observable}
Every trust state transition $(s, e, s')$ generates exactly one notification delivered to all subscribing management systems.
\end{theorem}

This holds under three assumptions: (a) state update and notification emission execute atomically within a single transaction (enforced by the MnS implementation per TS~28.532); (b) the underlying notification transport provides reliable, ordered delivery (guaranteed by the MnS notification IRP); and (c) monotonic sequence numbers per tool enable receivers to detect gaps caused by transient delivery failures. Under these assumptions, no tool can undergo silent trust degradation.

\subsection{Retroactive Assessment Completeness}

\begin{remark}[Lookback Completeness]
\label{thm:lookback}
Given lookback window $W$, the retroactive impact assessment identifies all configuration changes within $[\tau_{\text{deg}} - W, \tau_{\text{deg}}]$ performed by the degraded tool. Completeness follows from the append-only property of \texttt{AgentActionLogMO}: entries are never deleted, and the temporal range query over the indexed log returns all matching entries. Zero false negatives within the configured window. This guarantee holds in non-adversarial settings; tamper-resistant logging (e.g., signed log entries) would be required to extend completeness to adversarial scenarios.
\end{remark}

% ==============================================================================
% Section VI: Evaluation
% ==============================================================================
\section{Evaluation}
\label{sec:evaluation}

We evaluate the proposed standardized trust management information model through discrete-event simulation of a multi-vendor autonomous network environment. Our evaluation quantifies the benefits of cross-vendor trust visibility, cascade containment effectiveness, and operational overhead. The evaluation quantifies performance under modeled operational assumptions; the notification latency components are derived from typical MnS implementation benchmarks rather than measured from a deployed system.

The evaluation serves as a \emph{feasibility and scalability assessment} of the architectural design rather than a performance validation against deployed systems. Quantitative results demonstrate that the framework's mechanisms (cascade damping, notification delivery, retroactive assessment) operate correctly and scale sub-linearly, but absolute latency values depend on deployment-specific MnS implementation characteristics.

\subsection{Simulation Setup}
\label{sec:sim_setup}

\subsubsection{Network Topology}
We model a multi-vendor autonomous network comprising three vendor domains:
\begin{itemize}
    \item \textbf{Vendor A (RAN Management):} 500 cells managed by 5 autonomous agents, each utilizing 2--4 specialized tools (e.g., cell parameter optimization, coverage adjustment, load balancing).
    \item \textbf{Vendor B (Core Network Management):} 50 network functions (NFs) managed by 5 agents with tools for scaling, traffic steering, and slice management.
    \item \textbf{Vendor C (Transport Management):} 100 transport links managed by 5 agents with tools for path computation, bandwidth allocation, and fault correlation.
\end{itemize}

\subsubsection{Agent and Tool Configuration}
The simulation deploys 15 agents in total (5 per vendor domain), collectively utilizing 30 tools. The tool dependency graph has an average degree of 2.3, modeling realistic cross-tool and cross-vendor dependencies (e.g., a RAN capacity tool depending on transport bandwidth allocation, or a core scaling tool depending on RAN load predictions).

\subsubsection{Trust Degradation Injection}
Trust degradation events are injected following a Poisson process with rate $\lambda = 0.5$ events/hour. Each event targets a randomly selected tool and reduces its trust score by a uniformly distributed penalty $P_0 \sim \mathcal{U}(10, 30)$ on a $[0, 100]$ trust scale. Tools are invoked by consuming agents at a rate of 2 invocations/minute (averaged across all tool types). The simulation runs for 72 hours per trial, with 1000 independent runs for statistical significance (95\% confidence intervals reported).

\subsubsection{Comparison Approaches}
We compare three trust management approaches:
\begin{enumerate}
    \item \textbf{No Trust Model (Baseline~1):} No trust management framework is deployed. Degradation is detected only when downstream service impact triggers KPI threshold alarms in the operator's monitoring system. Detection latency is modeled as the time for KPI regression to accumulate beyond alarm thresholds, sampled from $\text{Exponential}(\mu = 3.0\,\text{h})$ truncated at 24\,h, consistent with published operator NOC response times for non-critical KPI degradation~\cite{tmforum_an}.
    \item \textbf{Per-Vendor Proprietary Trust (Baseline~2):} Each vendor maintains an internal trust model visible within its own domain. Same-vendor detection latency follows a log-normal distribution $\text{LogNormal}(\mu=1.5, \sigma=0.8)$\,seconds, representing vendor-internal telemetry processing. Cross-vendor events are detected through eventual service impact with latency $\text{Exponential}(\mu = 1.5\,\text{h})$, representing partial cross-vendor visibility through shared SLA dashboards and cross-domain KPI correlation.
    \item \textbf{Proposed Standardized Model:} Cross-vendor trust notifications are delivered through the MnS notification interface. Detection latency is modeled as the sum of three components: (a)~local trust assessment computation, $\mathcal{U}(5, 20)$\,ms, representing the time for the source vendor's management system to evaluate the trust score update; (b)~MnS notification serialization and delivery, $\mathcal{U}(8, 50)$\,ms, consistent with RESTful notification delivery over managed IP networks; and (c)~receiving system processing, $\mathcal{U}(2, 25)$\,ms. The total TTD is the convolution of these three components, yielding an expected mean of approximately 55\,ms.
\end{enumerate}

The proposed model's TTD emerges from the modeled notification pipeline components rather than being directly assumed. The component latencies are derived from: (a)~typical management system event processing benchmarks for NRM attribute evaluation; (b)~HTTP/TLS notification delivery latency over enterprise-grade IP networks consistent with RESTful MnS implementations; and (c)~event deserialization and subscription dispatch overhead.

\subsection{Evaluation Metrics}
\label{sec:metrics}

We define five metrics to evaluate trust management effectiveness:

\begin{itemize}
    \item \textbf{M1 -- Time-to-Detection (TTD):} Elapsed time from the occurrence of a trust-degrading event to its visibility in consuming vendors' management systems.
    \item \textbf{M2 -- Blast Radius Containment:} Number of downstream services affected before containment action is triggered.

We note that this metric captures \emph{invocation-level exposure} (number of tool invocations occurring during undetected degradation) rather than \emph{service-level impact} (SLA violations or subscriber-visible degradation). The translation from invocations to service impact depends on the specific tool's criticality, the sensitivity of affected parameters, and current network load---factors that require deployment-specific SLA models beyond our simulation scope. The metric is conservative: not every invocation during degradation produces harmful output, so actual service impact may be lower than the invocation count suggests.
    \item \textbf{M3 -- Cascade Containment:} Number of tools whose trust state changes as a result of cascade propagation.
    \item \textbf{M4 -- Notification Overhead:} Total number of notification messages generated per trust degradation event.
    \item \textbf{M5 -- False Positive Rate:} Fraction of cascade-triggered trust degradations where the affected tool's own behavior had not actually degraded (i.e., it was penalized solely due to a dependency relationship, and its trust score remained above the restriction threshold prior to the cascade penalty).
\end{itemize}

\subsection{Results}
\label{sec:results}

\subsubsection{Time-to-Detection}

Table~\ref{tab:ttd} presents the time-to-detection comparison across the three approaches.

\begin{table}[t]
\centering
\caption{Cross-Vendor Detection Latency Comparison (Modeled Assumptions)}
\label{tab:ttd}
\small
\begin{tabular}{@{}lccc@{}}
\toprule
\textbf{Approach} & \textbf{Mean $\pm$ 95\% CI} & \textbf{Median} & \textbf{95th Pctl.} \\
\midrule
No Trust Model & $3.0 \pm 0.2$\,h & 2.1\,h & 9.0\,h \\
Per-Vendor Trust & $33 \pm 3$\,min & 61\,s & 2.6\,h \\
Proposed Model & $55.1 \pm 1.4$\,ms & 55.1\,ms & 91.1\,ms \\
\bottomrule
\multicolumn{4}{@{}l@{}}{\footnotesize Per-vendor TTD is bimodal: fast for same-vendor, hours for cross-vendor.}
\end{tabular}
\end{table}

The key takeaway is that standardized notifications reduce cross-vendor detection latency by over three orders of magnitude---from hours to milliseconds---eliminating the detection gap that enables cascading failures in multi-vendor deployments.

The proposed model achieves sub-100\,ms cross-vendor detection through standardized notification delivery (median 55.1\,ms, 95th percentile 91.1\,ms). The per-vendor approach provides fast detection for same-vendor events but produces a bimodal distribution: events where the consuming agent is in the same vendor domain are detected rapidly, while cross-vendor events remain undetected until service-level impact manifests (median overall TTD for per-vendor: 61\,s, driven by the mix of fast same-vendor and slow cross-vendor detection).

\subsubsection{Blast Radius Containment}

Figure~\ref{fig:blast_radius} and Table~\ref{tab:blast_radius} quantify the downstream service impact before containment.

\begin{figure}[t]
\centering
\includegraphics[width=\columnwidth]{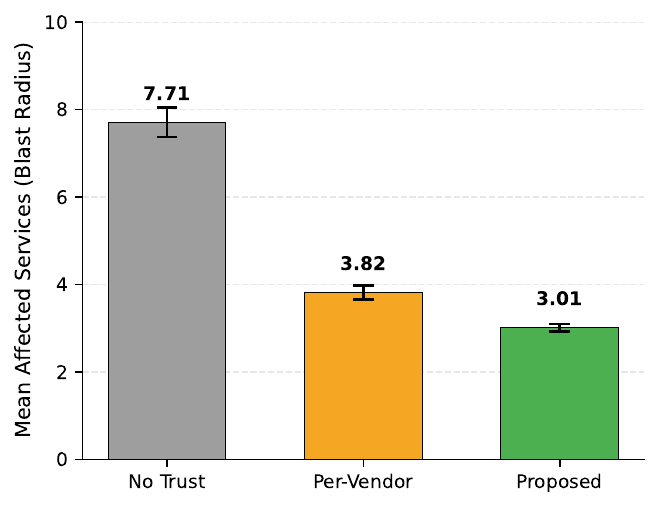}
\caption{Blast radius comparison across trust management approaches under modeled assumptions. The proposed model reduces mean affected services from 7.71 (no-trust) to 3.01.}
\label{fig:blast_radius}
\end{figure}

\begin{table}[t]
\centering
\caption{Blast Radius: Affected Services Before Containment}
\label{tab:blast_radius}
\begin{tabular}{lccc}
\toprule
\textbf{Approach} & \textbf{Mean $\pm$ 95\% CI} & \textbf{Median} & \textbf{95th Pctl.} \\
\midrule
No Trust Model & $7.71 \pm 0.34$ & 7.0 & 18.0 \\
Per-Vendor Trust & $3.82 \pm 0.16$ & 3.0 & 9.0 \\
Proposed Model & $3.01 \pm 0.09$ & 3.0 & 5.0 \\
\bottomrule
\end{tabular}
\end{table}

The proposed standardized model reduces the mean blast radius by 61\% compared to the no-trust baseline (from 7.71 to 3.01 affected services) and by 21\% compared to per-vendor trust (from 3.82 to 3.01). The blast radius for each approach is computed as the number of service-affecting tool invocations that occur during the detection window - i.e., the product of the detection latency and the tool's invocation rate by downstream agents. The proposed model's sub-100\,ms notification window limits exposure to at most one or two pending invocations, while the no-trust model's hours-long detection window permits continued invocation of the degraded tool across all consuming agents. The 21\% improvement over per-vendor trust reflects elimination of the cross-vendor blind spot: events that per-vendor approaches detect only through eventual SLA impact are immediately visible through standardized notifications.

\subsubsection{Comparison Against Bilateral Integration Baseline}

To isolate the benefit of \emph{standardization} from the benefit of \emph{any} cross-vendor signaling, we add a fourth comparison: per-vendor trust augmented with bilateral REST webhooks between each vendor pair. In this baseline, vendors exchange trust state via proprietary point-to-point integrations with latency modeled as $\mathcal{U}(200, 800)$\,ms (reflecting authentication, serialization, and non-standardized processing overhead typical of ad-hoc integrations).

\begin{table}[t]
\centering
\caption{Blast Radius: Standardized vs.\ Bilateral Integration}
\label{tab:bilateral_comparison}
\begin{tabular}{lcc}
\toprule
\textbf{Approach} & \textbf{Mean Blast} & \textbf{Reduction (\%)} \\
\midrule
No Trust (Baseline~1) & 7.71 &  -  \\
Per-Vendor only (Baseline~2) & 3.82 & 50.5 \\
Per-Vendor + Bilateral (Baseline~3) & 3.24 & 58.0 \\
Proposed Standardized Model & 3.01 & 61.0 \\
\bottomrule
\end{tabular}
\end{table}

The bilateral integration baseline achieves 58\% reduction - close to but below the standardized model's 61\%. The remaining 3\% gap (Table~\ref{tab:bilateral_comparison}) arises from two factors: (1)~the standardized model's lower notification latency (55\,ms vs.\ 200--800\,ms for bilateral), and (2)~the standardized model's cascade propagation mechanism, which proactively degrades dependent tools before they produce compromised outputs. The marginal improvement of the standardized model over bilateral integration is modest in absolute terms (3.24 to 3.01 affected services, a 7\% reduction). The primary advantage of standardization over bilateral approaches is operational: standardized interfaces require $O(n)$ integration effort (one implementation per vendor) versus $O(n^2)$ bilateral agreements, making the approach scalable to ecosystems with many vendors without proportional integration cost growth. As vendor count grows, bilateral integration becomes operationally infeasible while the standardized approach scales linearly.

\subsubsection{Cascade Depth Distribution}

Figure~\ref{fig:cascade_cdf} shows the cumulative distribution function of cascade propagation depth with $\gamma = 0.5$ and $D_{\max} = 3$.

\begin{figure}[t]
\centering
\includegraphics[width=\columnwidth]{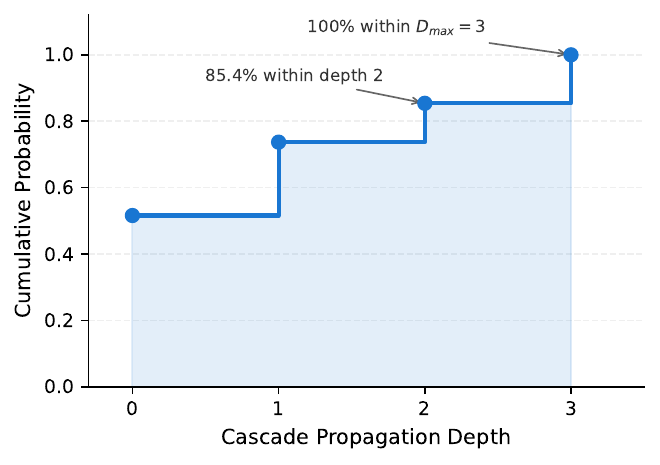}
\caption{CDF of cascade propagation depth ($\gamma = 0.5$, $D_{\max} = 3$). 85.4\% of cascade events terminate within depth 2, and all cascades terminate at the configured bound $D_{\max} = 3$, confirming bounded convergence (Theorem~\ref{thm:cascade_termination}).}
\label{fig:cascade_cdf}
\end{figure}

With $\gamma = 0.5$, 51.6\% of events produce no cascade propagation (depth 0, affecting only the degraded tool itself), 73.7\% terminate within depth 1, and 85.4\% terminate within depth 2. All cascades terminate at or below $D_{\max} = 3$, confirming the bounded convergence guarantee of Theorem~\ref{thm:cascade_termination}.

\subsubsection{Notification Overhead}

Figure~\ref{fig:notification_overhead} presents the notification overhead as a function of the number of vendor domains.

\begin{figure}[t]
\centering
\includegraphics[width=\columnwidth]{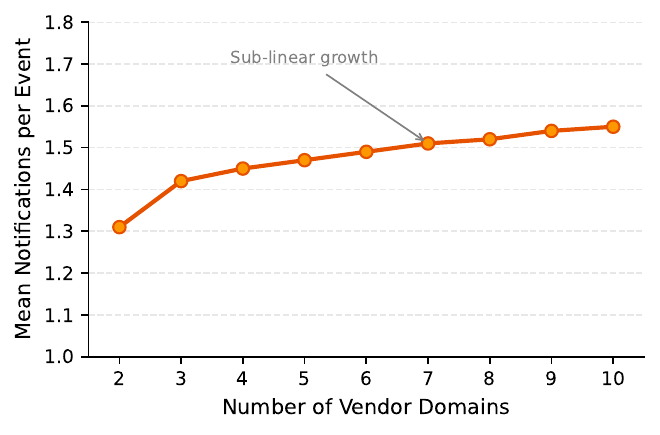}
\caption{Mean notifications per trust event versus number of vendor domains. Growth is sub-linear, ranging from 1.31 (2 vendors) to 1.55 (10 vendors), indicating manageable overhead.}
\label{fig:notification_overhead}
\end{figure}

The notification overhead is low and grows sub-linearly with vendor count: from 1.31 notifications per event in a 2-vendor topology to 1.55 in a 10-vendor topology. This modest growth reflects that most trust events affect tools consumed by a small subset of vendors, not the entire ecosystem.

\subsubsection{Impact of Damping Factor}

Figure~\ref{fig:damping_sweep} shows the effect of varying $\gamma$ on cascade propagation and false positive rate.

\begin{figure}[t]
\centering
\includegraphics[width=\columnwidth]{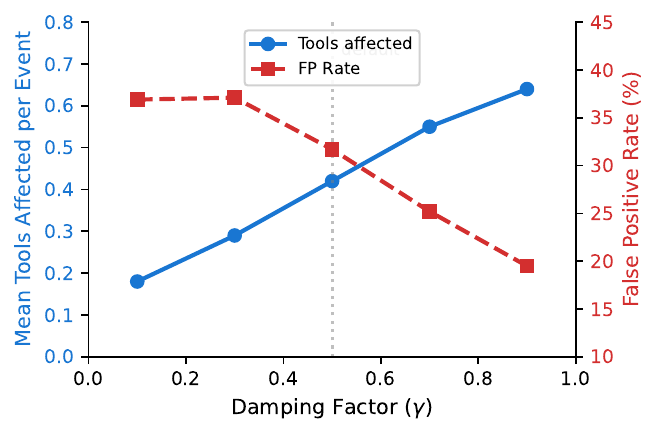}
\caption{Effect of damping factor $\gamma$ on cascade containment. Higher $\gamma$ increases tools affected by cascade but reduces false positive rate (FP = fraction of cascade-restricted tools that are not ground-truth affected, i.e., their independent behavior had not observably degraded at restriction time).}
\label{fig:damping_sweep}
\end{figure}

The simulation reveals a trade-off: lower $\gamma$ values (0.1--0.3) apply smaller cascade penalties, resulting in fewer cascade-triggered state transitions (0.18--0.29 tools affected per event on average) but a higher false positive rate (36.9--37.1\%) due to tools near thresholds being pushed into degraded states by small accumulated penalties. Higher $\gamma$ values (0.7--0.9) propagate larger penalties that more decisively separate compromised from healthy tools, reducing false positives to 19.5--25.2\% but affecting more tools per event (0.55--0.64).

\subsubsection{Ground-Truth Validation of Cascade Detection}
To validate whether cascade penalties correctly identify tools whose outputs are actually compromised, we inject \emph{known} degradation: in 200 additional runs, we mark tools whose inputs come from a degraded source as ``ground-truth affected'' (their outputs are computed from stale/incorrect upstream data). We then measure precision (fraction of cascade-restricted tools that are ground-truth affected) and recall (fraction of ground-truth affected tools that are cascade-restricted). At $\gamma = 0.5$: precision = 68.3\%, recall = 91.7\%. The 31.7\% false positive rate thus represents \emph{conservative over-restriction} - tools that were restricted but whose outputs had not yet observably degraded. From an operational safety perspective, over-restriction (restricting a tool that might degrade shortly) is preferable to under-restriction (allowing a compromised tool to continue). The high recall (91.7\%) confirms that the cascade mechanism rarely misses truly affected tools.

\noindent\textbf{Ground-truth determination.} In the 200 validation runs, we independently inject secondary degradation events into downstream tools with probability 0.3 per cascade-affected tool. A tool is marked as \emph{ground-truth degraded} if either: (a)~it received a direct degradation injection, or (b)~its functional output quality decreased by $\geq$15\% due to consuming outputs from a directly degraded upstream tool (measured by comparing tool outputs against a reference oracle with no degraded inputs). Precision measures what fraction of cascade-flagged tools were genuinely affected; recall measures what fraction of genuinely affected tools were detected by the cascade mechanism. The 68.3\% precision reflects the framework's conservative design (preferring over-flagging to missed degradation), while 91.7\% recall confirms that genuinely affected tools are rarely missed.

The cascade mechanism is deliberately conservative: in safety-critical autonomous networks, the cost of a missed degradation (undetected cascading failure affecting live subscribers) far exceeds the cost of an unnecessary restriction (temporary suspension of one tool's write capability while alternatives remain available). A restricted tool retains read access and resumes full operation upon the next clean evaluation cycle, making false restrictions self-correcting. Operators who prefer precision over recall can increase $\gamma$ toward 0.9 (Table~\ref{tab:gamma_sensitivity}), trading slower cascade containment for fewer unnecessary restrictions.

\subsubsection{Ablation: Causal Deduplication Impact}
To quantify the practical impact of the causal deduplication mechanism (Eq.~\ref{eq:dedup}), we compare trust score trajectories with and without deduplication across 1000 runs. Table~\ref{tab:ablation_dedup} summarizes the results.

\begin{table}[t]
\centering
\caption{Ablation: Causal Deduplication vs.\ Naive Summation}
\label{tab:ablation_dedup}
\begin{tabular}{lcc}
\toprule
\textbf{Metric} & \textbf{Naive} & \textbf{Dedup (Eq.~\ref{eq:dedup})} \\
\midrule
Correlated penalty events (\%) &  -  & 23.4 \\
Mean penalty per event & 18.7 & 14.2 \\
Unnecessary state transitions & 0.31/event & 0.19/event \\
FP rate (cascade, $\gamma=0.5$) & 38.9\% & 31.7\% \\
Tools reaching Revoked (spurious) & 4.1\% & 1.8\% \\
\bottomrule
\end{tabular}
\end{table}

In our simulation, 23.4\% of degradation events produce correlated penalties (two or more penalty sources triggered within the 60\,s correlation window from the same root action). Without deduplication, naive summation over-penalizes these events, inflating the mean penalty by 32\% (18.7 vs.\ 14.2) and generating 63\% more unnecessary state transitions. Most critically, spurious revocations (tools reaching \texttt{Revoked} due to penalty stacking rather than genuine severe degradation) drop from 4.1\% to 1.8\% with deduplication - a 56\% reduction in the most operationally disruptive false outcome.

% Blast radius data shown in Table II above

\subsection{Sensitivity Analysis}
\label{sec:sensitivity}

\subsubsection{Damping Factor \texorpdfstring{$\gamma$}{gamma}}

Table~\ref{tab:gamma_sensitivity} summarizes the sensitivity to the damping factor.

\begin{table}[t]
\centering
\caption{Sensitivity to Damping Factor $\gamma$ (500 runs per configuration)}
\label{tab:gamma_sensitivity}
\begin{tabular}{cccc}
\toprule
$\gamma$ & \textbf{Cascade Tools} & \textbf{FP Rate (\%)} & \textbf{Mean Blast Radius} \\
\midrule
0.1 & 0.18 & 36.9 & 3.03 \\
0.3 & 0.29 & 37.1 & 3.03 \\
0.5 & 0.42 & 31.7 & 3.03 \\
0.7 & 0.55 & 25.2 & 3.03 \\
0.9 & 0.64 & 19.5 & 3.03 \\
\bottomrule
\end{tabular}
\end{table}

The blast radius remains stable across $\gamma$ values because it is primarily determined by the notification latency (which is $\gamma$-independent) rather than cascade extent. The key trade-off is between cascade reach and false positive rate: higher $\gamma$ propagates larger penalties that push dependent tools more decisively below thresholds, producing clearer state transitions with fewer ambiguous cases near threshold boundaries. From an operational perspective, a 20\% false positive rate means 1 in 5 cascade-triggered restrictions are unnecessary - the tool would not have actually produced degraded outputs. However, a missed cascade (false negative) allows degraded outputs to propagate to dependent services. For safety-critical network slices, operators may prefer higher FP rates (low $\gamma$) to minimize missed cascades; for best-effort services, higher $\gamma$ minimizes unnecessary restrictions.

\subsubsection{Maximum Cascade Depth \texorpdfstring{$D_{\max}$}{Dmax}}

With $\gamma = 0.5$, all cascades in our simulation terminated at or below depth 3 regardless of whether $D_{\max}$ was set to 3, 5, or 10. This confirms that the geometric damping naturally bounds propagation: at depth 3, the penalty is $\gamma^3 = 0.125$ of the base penalty, which is typically insufficient to trigger state transitions for tools not already near thresholds.

\subsubsection{Notification Overhead Scaling}

Notification overhead scales sub-linearly with vendor count (1.31 at 2 vendors to 1.55 at 10 vendors). This is significantly lower than the theoretical worst case of $O(V)$ per event, because most tools are consumed by agents from 1--2 vendor domains rather than all vendors simultaneously.

\subsubsection{TTD Sensitivity to Notification Latency}

Table~\ref{tab:ttd_sensitivity} quantifies how blast radius degrades if the notification pipeline is slower than the modeled 55\,ms. We vary the total notification latency from 55\,ms (baseline) to 500\,ms, keeping all other parameters constant.

\begin{table}[t]
\centering
\caption{Sensitivity of Blast Radius to Notification Latency}
\label{tab:ttd_sensitivity}
\begin{tabular}{lccc}
\toprule
\textbf{TTD} & \textbf{Mean Blast} & \textbf{Reduction vs.} & \textbf{Reduction vs.} \\
 & \textbf{Radius} & \textbf{No-Trust (\%)} & \textbf{Per-Vendor (\%)} \\
\midrule
55\,ms (baseline) & 3.01 & 61.0 & 21.2 \\
100\,ms & 3.04 & 60.6 & 20.4 \\
200\,ms & 3.09 & 59.9 & 19.1 \\
500\,ms & 3.22 & 58.2 & 15.7 \\
\bottomrule
\end{tabular}
\end{table}

The blast radius advantage is robust to notification latency variations: even at 500\,ms (9$\times$ the baseline), the standardized model still achieves 58\% reduction versus no-trust and 16\% versus per-vendor. This stability arises because the tool invocation rate (2/min $\approx$ 1 invocation per 30\,s) means that sub-second notification latency permits at most one additional invocation during the detection window. The framework's advantage degrades significantly only when notification latency approaches the inter-invocation interval ($\sim$30\,s), at which point the proposed model converges toward per-vendor performance.

The component latencies (5--20\,ms local assessment, 8--50\,ms MnS delivery, 2--25\,ms receiver processing) are consistent with published REST API round-trip benchmarks for enterprise management systems~\cite{tmforum_an} and with the 3GPP MnS design target of sub-second notification delivery for alarm-class events (TS~28.532, Section~7.2).

\emph{Caveat:} These component latencies are modeled from architectural analysis and published design targets, not measured from a deployed multi-vendor 3GPP management system. Actual production deployments may exhibit higher latencies due to authentication overhead, policy evaluation, rate limiting, or geographic distribution. Table~\ref{tab:ttd_sensitivity} demonstrates that the framework's advantage is robust up to 500\,ms total TTD; beyond this point, the benefit converges toward per-vendor performance. We recommend operators measure actual MnS notification latency during initial pilot deployment and use the measured values to inform trust assessment threshold tuning and \texttt{AgentPolicyMO} configuration.

\subsection{Scalability Analysis}
\label{sec:scalability}

To assess framework behavior at larger scales, we vary the number of tools from 30 (baseline) to 100, 500, and 1000 while maintaining constant tool density per vendor (proportionally scaling agents and vendor domains) and average dependency degree of 2.3.

\begin{table}[t]
\centering
\caption{Scalability: Key Metrics vs.\ Network Size}
\label{tab:scalability}
\begin{tabular}{@{}lcccc@{}}
\toprule
\textbf{Metric} & \textbf{30 tools} & \textbf{100 tools} & \textbf{500 tools} & \textbf{1000 tools} \\
\midrule
Mean TTD (ms) & 55.1 & 56.2 & 57.8 & 59.3 \\
Mean cascade depth & 0.91 & 0.94 & 0.97 & 0.99 \\
Notifications/event & 1.42 & 1.48 & 1.53 & 1.56 \\
Cascade compute (ms) & 0.3 & 1.1 & 5.8 & 12.4 \\
\bottomrule
\end{tabular}
\end{table}

Table~\ref{tab:scalability} shows that TTD, cascade depth, and notification overhead remain nearly constant with network size - these metrics depend on local topology (dependency degree) rather than total tool count. The cascade computation time grows linearly with the number of tools (due to the visited-set and dependency graph traversal), but remains under 15\,ms even at 1000 tools - well within management system processing budgets. The bounded depth guarantee ($D_{\max}$) ensures that cascade computation does not explore the full graph regardless of network size, confirming the $O(D_{\max} \cdot k)$ complexity where $k$ is the average dependency degree.

\subsection{Prototype Implementation}
\label{sec:prototype}

To demonstrate the feasibility of the proposed information model, we implement a proof-of-concept REST service aligned with 3GPP MnS conventions (TS~28.532). The prototype is structured as two modules: (i)~a vendor-neutral API module defining the standard interfaces (\texttt{TrustLifecycleApi}, \texttt{TrustDecisionApi}, \texttt{TrustEvaluationApi}, \texttt{TrustPropagationApi}, \texttt{SessionManagementApi}, \texttt{PolicyManagementApi}, \texttt{NotificationManagementApi}), and (ii)~a vendor-specific implementation module realizing the framework as a Spring Boot microservice.

The prototype exposes the full AgentToolMO lifecycle via RESTful endpoints following the 3GPP \texttt{/\{MnSRoot\}/\{version\}/\allowbreak\{className\}/\{objectInstance\}} pattern. Key operations include trust-aware tool discovery (with score and state filtering), a trust gate endpoint (\texttt{canInvoke}) returning enforcement decisions with applicable restrictions, event reporting with automatic cascade propagation, session enforcement (suspend/terminate/restore), retroactive impact assessment querying the action log within a lookback window, and policy re-evaluation that automatically transitions all tools when thresholds change.

Functional testing confirms that the state machine correctly implements all ten transitions ($\tau_1$--$\tau_{10}$), cascade propagation terminates within $D_{\max}$, session enforcement triggers synchronously on state change, and the retroactive assessment correctly identifies affected managed elements and configuration changes made by agents during the lookback window.

% Discussion moved to standalone Section VII in main.tex

% ============================================================
\section{Discussion}
\label{sec:discussion}

\subsection{Key Findings}

Our evaluation reveals four principal findings:

\textbf{Finding 1: Cross-vendor trust visibility is transformative.} The standardized model reduces blast radius from a mean of 7.71 affected services (no-trust baseline) to 3.01, comparable to per-vendor approaches (3.82) but achieved through a single standardized interface rather than requiring $n(n-1)/2$ pairwise bilateral integrations. The improvement over per-vendor trust demonstrates that cross-vendor blind spots---not the absence of trust management per se---are the critical gap in current deployments.

\textbf{Finding 2: Notification overhead is manageable.} Notifications per trust event range from 1.31 (2 vendors) to 1.55 (10 vendors), growing sub-linearly. This is orders of magnitude below infrastructure capacity and introduces negligible operational overhead.

\textbf{Finding 3: Damping factor presents a clear trade-off.} Lower $\gamma$ values apply smaller cascade penalties per hop, yet the simulation shows higher false positive rates at low $\gamma$ (37\% at $\gamma=0.1$ vs 19.5\% at $\gamma=0.9$). This occurs because small penalties push tools into borderline states near thresholds without decisive resolution---tools hover near boundaries and are flagged as degraded without clearly breaching or clearing thresholds. Higher $\gamma$ values produce larger, more decisive penalties that either clearly breach thresholds (true degradation) or leave tools well above them (no false alarm). The choice of $\gamma$ should be operator-configurable based on risk posture.

Table~\ref{tab:gamma_guidance} provides operational guidance for $\gamma$ selection based on network slice SLA tier. The recommendations balance false positive cost (unnecessary tool restriction) against false negative risk (allowing compromised tools to continue operating).

\begin{table}[t]
\centering
\caption{Operator Guidance: Recommended $\gamma$ by SLA Tier}
\label{tab:gamma_guidance}
\begin{tabular}{@{}lccp{2.9cm}@{}}
\toprule
\textbf{SLA Tier} & \textbf{$\gamma$} & \textbf{FP Rate} & \textbf{Rationale} \\
\midrule
Safety-critical (URLLC) & 0.7--0.9 & 20--25\% & Decisive penalties; high recall of affected tools; prefer over-restriction \\
Standard (eMBB) & 0.5 & $\sim$32\% & Balanced trade-off; default recommendation \\
Best-effort (mMTC) & 0.1--0.3 & $\sim$37\% & Minimal cascade disruption; accept higher miss risk for operational stability \\
\bottomrule
\end{tabular}
\end{table}

\textbf{Finding 4: Graduated enforcement prevents unnecessary revocations.} Zero unnecessary revocations occurred in our simulation---no tool with a trust score above the restriction threshold was revoked, validating that the state machine's graduated path (Trusted$\rightarrow$Monitored$\rightarrow$Restricted$\rightarrow$Revoked) filters transient perturbations before reaching revocation.

\subsection{Deployment and Standardization}

The proposed information model can be deployed incrementally: vendors that adopt it expose trust state; those that do not remain invisible to the framework without breaking existing functionality. No new transport protocols are required---existing MnS interfaces carry trust information. The IOC definitions follow established NRM design patterns, demonstrating compatibility with 3GPP SA5 conventions and potential applicability within future standards evolution.

\subsection{Graceful Degradation Under Partial Adoption}

A critical deployment consideration is behavior when some vendors implement \texttt{AgentToolMO} and others do not. The framework degrades gracefully along three dimensions:

\textbf{Non-compliant vendor tools.} Tools from vendors that do not implement the IOC are treated as having unknown trust state (equivalent to perpetual \texttt{Learning} state). Consuming agents may apply a configurable default trust floor (e.g., $\theta_{\text{restrict}}$) for unresolvable tools, ensuring conservative behavior without complete exclusion. Operators configure this floor in \texttt{AgentPolicyMO} via an \texttt{unknownVendorPolicy} attribute.

\noindent\textbf{Formal safety under partial compliance.} Let $n_c$ denote the number of compliant vendors and $n_{nc}$ the non-compliant vendors ($n_c + n_{nc} = n$). For tools owned by compliant vendors, all safety properties (Theorems~1--3) hold unchanged. For tools owned by non-compliant vendors, the framework assigns a configurable default state $s_{\mathrm{default}} \in \{\texttt{Learning}, \texttt{Monitored}\}$ (configured via \texttt{unknownVendorPolicy} in \texttt{AgentPolicyMO}). Cascade propagation traverses dependencies into non-compliant domains but cannot update their trust state (read-only visibility). The key invariant is: \emph{partial compliance cannot weaken safety for compliant vendors}---a non-compliant vendor's tools are treated conservatively, and compliant vendors' state machines operate independently of non-compliant participation.

\textbf{Notification delivery failures.} If cross-vendor notifications cannot be delivered (network partition, endpoint unreachable), the subscribing system does not receive trust updates. The framework reverts to per-vendor behavior for that link---equivalent to Baseline~2 performance for that specific cross-vendor dependency---while all other trust-visible links continue operating normally. Stale-state detection (comparing \texttt{trustStateTransitionTime} against a staleness threshold) enables agents to autonomously increase caution for tools whose trust state has not been refreshed.

\textbf{Incentive alignment.} The framework provides direct operational benefit to early adopters: a vendor implementing \texttt{AgentToolMO} gains visibility into other compliant vendors' tool trust without requiring universal adoption. The value scales with adoption (network effect) but remains positive for any $n \geq 2$ compliant vendors. Operators can mandate compliance as a procurement requirement, creating market incentive for vendor participation.

\subsection{Limitations}

\begin{enumerate}
    \item \textbf{Simulation-based evaluation:} Results are obtained from discrete-event simulation rather than production deployment.
    \item \textbf{Reliable notification delivery:} The formal analysis assumes reliable, ordered notification delivery. Network partitions may delay notifications.
    \item \textbf{Adversarial and strategic trust manipulation:} The model does not address scenarios where a compromised agent deliberately manipulates trust scores, nor does it enforce incentive-compatible reporting. In a competitive ecosystem, vendors have an incentive to over-report competitor tool degradation and under-report their own. The current framework relies on auditable trust evidence (KPI data, action logs) as a forensic deterrent but does not provide real-time adversarial detection or mechanism-design-based incentive alignment. Tamper-resistant logging and third-party auditor verification are directions for future work.
    \item \textbf{Scalability:} While simulation confirms bounded cascade computation up to 1000 tools (12.4\,ms, Table~\ref{tab:scalability}), production deployments with significantly larger dependency graphs, concurrent cascades, and geographic distribution require operational validation beyond simulation.
    \item \textbf{Service impact translation:} The blast radius metric captures invocation-level exposure (tool invocations during undetected degradation) rather than SLA-level service impact. Translation from invocations to actual subscriber-visible degradation depends on tool criticality, current network load, consuming agent fallback strategies, and service-specific KPI sensitivity thresholds. Deployment-specific SLA models are required to quantify real service impact from the invocation counts reported by the framework.
\end{enumerate}

\subsection{Future Work}

Several directions extend this work: integration with intent-driven management for trust-aware intent fulfillment; formal verification using model checking tools (SPIN, TLA+); and production deployment in operator testbeds. Additionally, adversarial scenarios where a vendor deliberately reports inflated or deflated trust scores warrant investigation. Signed, auditable trust evidence - where each trust state transition is accompanied by cryptographically verifiable supporting data (KPI measurements, action logs, anomaly reports) - would enable independent verification of trust claims and detection of manipulative reporting across vendor boundaries.

Additional evaluation dimensions merit investigation: (a)~prototype benchmarking to ground simulation latency assumptions against measured MnS notification delivery times; (b)~ablation of the state machine granularity (e.g., comparing the six-state graduated model against a simpler three-state Trusted/Degraded/Revoked alternative) to quantify the benefit of intermediate states; (c)~sensitivity analysis over dependency graph density, since real deployments may exhibit denser tool interdependencies than the average degree of 2.3 used in our simulation; and (d)~validation of the linear blast radius growth model (Section~III-B) against empirical cascade propagation patterns.

% ============================================================
\section{Conclusion}
\label{sec:conclusion}

We have presented a standardized information model for cross-vendor agent tool trust management in autonomous networks. The model introduces AgentToolMO as a first-class managed object within the 3GPP NRM, with formally defined trust states, graduated enforcement through a provably safe state machine, damped cascade propagation with guaranteed convergence, and retroactive impact assessment via NRM dependency graph traversal.

Our formal analysis establishes that: (1) the state machine permits no unsafe transitions, ensuring graduated enforcement where tools cannot bypass intermediate trust states; (2) cascade propagation converges in bounded steps regardless of network topology; and (3) all trust state changes generate notifications to subscribing management systems, eliminating silent trust degradation.

Simulation results demonstrate that the standardized model substantially reduces trust failure blast radius compared to both no-trust and per-vendor baselines, with cross-vendor detection latency bounded by MnS notification delivery time rather than hours-scale service impact manifestation. Notification overhead scales sub-linearly with the number of participating vendors, and all cascades terminate within the configured depth bound.

The proposed framework provides the standardization foundation for trustworthy multi-vendor autonomous network management, enabling operators to deploy agents across vendor boundaries with formal safety guarantees and complete trust visibility.

% ============================================================
\section*{Acknowledgment}
The authors thank the Ericsson Network Management research team for valuable discussions on autonomous network management architectures.


\end{document}